\documentclass[11pt]{article}

\usepackage{graphicx}
\usepackage{latexsym}
\usepackage{amsmath}
\usepackage{amssymb}
\usepackage{natbib}
\usepackage{url}
\usepackage{amsthm}

\newcommand{\bA}{\mbox{\boldmath {$A$}}}

\newcommand{\bB}{\mbox{\boldmath {$B$}}}

\newcommand{\bD}{\mbox{\boldmath {$D$}}}
\newcommand{\be}{\mbox{\boldmath {$e$}}}

\newcommand{\bh}{\mbox{\boldmath {$h$}}}
\newcommand{\bH}{\mbox{\boldmath {$H$}}}

\newcommand{\bI}{\mbox{\boldmath {$I$}}}

\newcommand{\bM}{\mbox{\boldmath {$M$}}}

\newcommand{\bO}{\mbox{\boldmath {$O$}}}

\newcommand{\bS}{\mbox{\boldmath {$S$}}}

\newcommand{\bx}{\mbox{\boldmath {$x$}}}

\newcommand{\by}{\mbox{\boldmath {$y$}}}

\newcommand{\bze}{\mbox{\boldmath {$0$}}}

\newcommand{\bmu}{\mbox{\boldmath $ \mu $}}
\newcommand{\bSig}{\mbox{\boldmath $ \Sigma $}}

\newcommand{\bSigma}{\mbox{\boldmath $ \Sigma $}}

\newcommand{\bLam}{\mbox{\boldmath $ \Lambda $}}

\newcommand{\bGamma}{\mbox{\boldmath {$\Gamma$}}}
\newcommand{\bgam}{\mbox{\boldmath $\gamma$}}

\newcommand{\tr}{\mbox{tr}}
\newcommand{\Var}{\mbox{Var}}
\newcommand{\argmin}{\mathop{\rm argmin}\limits}

\numberwithin{equation}{section}
\newtheorem{thm}{Theorem}[section]
\newtheorem{cor}{Corollary}[section]
\newtheorem{pro}{Proposition}[section]
\newtheorem{rem}{Remark}
\newtheorem{lem}{Lemma}[section]

\long\def\symbolfootnote[#1]#2{\begingroup%
\def\thefootnote{\fnsymbol{footnote}}\footnote[#1]{#2}\endgroup}

\begin{document}

\begin{center}
\Large
{\bf High-dimensional quadratic classifiers in non-sparse settings}
\end{center}
\begin{center}
\vskip 0.5cm
\textbf{\large Makoto Aoshima and Kazuyoshi Yata} \\
Institute of Mathematics, University of Tsukuba, Ibaraki, Japan \\[-1cm]
\end{center}
\symbolfootnote[0]{\normalsize Address correspondence to Makoto Aoshima, 
Institute of Mathematics, University of Tsukuba, Ibaraki 305-8571, Japan; 
Fax: +81-298-53-6501; E-mail: aoshima@math.tsukuba.ac.jp}

\begin{abstract}
We consider high-dimensional quadratic classifiers in non-sparse settings. 
The target of classification rules is not Bayes error rates in the context.
The classifier based on the Mahalanobis distance does not always give a preferable performance even if the populations are normal distributions having known covariance matrices.
The quadratic classifiers proposed in this paper draw information about heterogeneity effectively through both the differences of expanding mean vectors and covariance matrices. 
We show that they hold a consistency property in which misclassification rates tend to zero as the dimension goes to infinity under non-sparse settings. 
We verify that they are asymptotically distributed as a normal distribution under certain conditions. 
We also propose a quadratic classifier after feature selection by using both the differences of mean vectors and covariance matrices.
Finally, we discuss performances of the classifiers in actual data analyses.
The proposed classifiers achieve highly accurate classification with very low computational costs. 
\\[2mm]
{\small \noindent\textbf{Keywords:} Bayes error rate; Discriminant analysis; Feature selection; Heterogeneity; Large $p$ small $n$}
\end{abstract}

\section{Introduction}
Globally, there is an ever increasing need for fast, accurate and cost effective analysis of high-dimensional data in many fields, including academia, medicine and business.
However, existing classifiers for high-dimensional data are often complex, time consuming and have no guarantee of accuracy.
In this paper we hope to provide better options.
A common feature of high-dimensional data is that the data dimension is high, however, the sample size is relatively low. 
This is the so-called ``HDLSS" or ``large $p$, small $n$" data situation where $p/n\to\infty$; here $p$ is the data dimension and $n$ is the sample size. 
Suppose we have independent and $p$-variate two populations, $\pi_i,\ i=1,2$, having an unknown mean vector $\bmu_i=(\mu_{i1},...,\mu_{ip})^T$ and unknown covariance matrix $\bSigma_i(>\bO)$ for each $i$. 
Let 
$$
\bmu_{12}=\bmu_1-\bmu_2=(\mu_{121},...,\mu_{12p})^T \quad \mbox{and}\quad \bSigma_{12}=\bSigma_{1}-\bSigma_{2}.
$$ 
We assume that $\limsup_{p\to \infty} |\mu_{12j}|<\infty$ for all $j$.
Note that $\limsup_{p\to \infty} ||\bmu_{12}||^2/p<\infty$, where $||\cdot||$ denotes the Euclidean norm. 
Let $\sigma_{i(j)}$ be the $j$-th diagonal element of $\bSig_i$ for $j=1,...,p\ (i=1,2)$. 
We assume that $\sigma_{i(j)}\in (0,\infty)$ as $p\to \infty$ for all $i,j$. 
Here, for a function, $f(\cdot)$, ``$f(p) \in (0, \infty)$ as $p\to \infty$" implies that $\liminf_{p\to \infty}f(p)>0$ and $\limsup_{p\to \infty}f(p)<\infty$. 
Then, it holds that $\tr(\bSig_i)/p\in (0, \infty)$ as $p\to \infty$ for $i=1,2$.
We do not assume $\bSigma_1=\bSigma_2$.
The eigen-decomposition of $\bSigma_{i}$ is given by $\bSigma_{i}=\bH_{i}\bLam_{i}\bH_{i}^T$, where $\bLam_{i}=\mbox{diag}(\lambda_{i1},...,\lambda_{ip})$ is a diagonal matrix of eigenvalues, $\lambda_{i1}\ge \cdots \ge \lambda_{ip}> 0$, and $\bH_{i}=[\bh_{i1},...,\bh_{ip}]$ is an orthogonal matrix of the corresponding eigenvectors. 
We have independent and identically distributed (i.i.d.) observations, $\bx_{i1},...,\bx_{in_i}$, from each $\pi_i$, where $\bx_{ik}=(x_{i1k},...,x_{ipk})^T,\ k=1,...,n_i$. 
We assume $n_i\ge 2,\ i=1,2$. 
Let $n_{\min}=\min\{n_1,n_2\}$. 
We estimate $\bmu_i$ and $\bSigma_i$ by 
$\overline{\bx}_{in_i}=(\overline{x}_{i1n_i},...,\overline{x}_{ipn_i})^T=\sum_{k=1}^{n_i}{\bx_{ik}}/{n_i}$ 
and $\bS_{in_i}=\sum_{k=1}^{n_i}(\bx_{ik}-\overline{\bx}_{in_i})(\bx_{ik}-\overline{\bx}_{in_i})^T/(n_i-1)$.  
Let $s_{in_i(j)}$ be the $j$-th diagonal element of $\bS_{in_i}$ for $j=1,...,p\ (i=1,2)$. 

In this paper, we consider high-dimensional quadratic classifiers in non-sparse settings. 
Let $\bx_0=(x_{01},...,x_{0p})^T$ be an observation vector of an individual belonging to one of the two 
populations.  
Let $|\bM|$ be the determinant of a square matrix $\bM$. 
When $\pi_i$s are Gaussian, a Bayes optimal rule is given as follows: 
One classifies the individual into $\pi_1$ if
\begin{align}
(\bx_0-\bmu_1)^T\bSig_{1}^{-1}(\bx_0-\bmu_1)
-\log|\bSig_{2}\bSig_{1}^{-1}| <(\bx_0-\bmu_2)^T\bSig_{2}^{-1}(\bx_0-\bmu_2)
\label{1.1}
\end{align}
and into $\pi_2$ otherwise. 
Since $\bmu_i$s and $\bSig_i$s are unknown, one usually considers the following typical classifier: 
$$
(\bx_0-\overline{\bx}_{1n_1})^T\bS_{1n_1}^{-1}(\bx_0-\overline{\bx}_{1n_1})-\log|\bS_{2n_2}\bS_{1n_1}^{-1}|< (\bx_0-\overline{\bx}_{2n_2})^T\bS_{2n_2}^{-1}(\bx_0-\overline{\bx}_{2n_2}).
$$
The classifier usually converges to the Bayes optimal classifier when $n_{\min}\to \infty$ while $p$ is fixed or $n_{\min}/p\to \infty$. 
However, in the HDLSS context, the inverse matrix of $\bS_{in_i}$ does not exist. 
When $\bSigma_1=\bSigma_2$, \cite{Bickel:2004} considered an inverse matrix defined by only diagonal elements of the pooled sample covariance matrix. 
\cite{Fan:2008} considered a classification after feature selection. 
\cite{Fan:2012} proposed the regularized optimal affine discriminant (ROAD). 
When $\bSigma_1\neq \bSigma_2$, \cite{Dudoit:2002} considered an inverse matrix defined by only diagonal elements of $\bS_{in_i}$. 
\cite{Aoshima:2011} considered using $\{\tr(\bS_{in_i})/p\}^{-1}\bI_p$ instead of $\bS_{in_i}^{-1}$ from a geometrical background of HDLSS data and proposed the geometric classifier. 
Here, $\bI_p$ denotes the identity matrix of dimension $p$. 
\cite{Hall:2005} and \cite{Marron:2007} considered distance weighted classifiers.
\cite{Chan:2009} and \cite{Aoshima:2014} considered distance-based classifiers and \cite{Aoshima:2014} gave the misclassification rate adjusted classifier for multiclass, high-dimensional data whose misclassification rates are no more than specified thresholds. 

Recently, \cite{Cai:2011b}, \cite{Shao:2011} and \cite{Li:2015} gave sparse linear or quadratic classification rules for high-dimensional data. 
They showed that their classification rules have Bayes error rates when $\pi_i$s are Gaussian.
They assumed that $\lambda_{ij}$s are bounded under some sparsity conditions such as $\bmu_{12}$, $\bSig_i$s and $\bSig_{12}$ (or $\bSig_i^{-1}$s and $\bSig_1^{-1}-\bSig_2^{-1}$) are sparse. 
For example, when $\bSig_1=\bSig_2\ (=\bSig,\ \mbox{say})$, the error rate of their classification rules is given by $\Phi(-\Delta_{MD}^{1/2}/2)+o(1)$ as $p\to \infty$, where $\Delta_{MD}=\bmu_{12}^T\bSig^{-1}\bmu_{12}$ that is the Mahalanobis distance and $\Phi(\cdot )$ denotes the cumulative distribution function of the standard normal distribution.
Here, $\Phi(-\Delta_{MD}^{1/2}/2)$ is the Bayes error rate.

In this paper, we investigate quadratic classifiers from a perspective that is different from the sparse discriminant analysis. 
We {\it do not assume that $\bmu_{12}$, $\bSig_i$s and $\bSig_{12}$ are sparse}.
In such a context, the target of classification rules is not Bayes error rates as in $\Phi(-\Delta_{MD}^{1/2}/2)+o(1)$ as $p\to\infty$. 
We consider a consistency property such as misclassification rates tend to $0$ as $p$ increases, i.e., 
$$
e(i)\to 0 \ \ \mbox{as $p\to \infty$ for $i=1,2$},
$$
where $e(i)$ denotes the error rate of misclassifying an individual from $\pi_i$ into the other class. 
For example, if one can assume that $\pi_i$s are Gaussian and $\bSig_1=\bSig_2$, the Bayes rule by (\ref{1.1}) has such a consistency property when $\Delta_{MD} \to \infty $ as $p\to \infty$. 
It is likely that $\Delta_{MD} \to \infty $ as $p\to \infty$ when $\bmu_{12}$ is non-sparse in the sense that $||\bmu_{12}||\to \infty$ as $p\to \infty$.
We emphasize that such non-sparse situations often occur in high-dimensional settings. 
For example, see \cite{Hall:2005} or (\ref{6.1}), (\ref{6.2}) and Table~\ref{Tab2} in Section 6.
We will show that quadratic classifiers hold the consistency property when $\bmu_{12}$ or $\bSig_{12}$ is non-sparse such as 
$||\bmu_{12}|| \to \infty$ or $||\bSig_{12}||_F\to \infty$ as $p\to \infty$, where $||\cdot||_F$ is the Frobenius norm.
 
In this paper, we consider the following function of $\bA_i$ to discriminate $\pi_i$s in general: 
\begin{align}
W_i(\bA_i)=&(\bx_0-\overline{\bx}_{in_i})^T\bA_i(\bx_0-\overline{\bx}_{in_i})-\tr(\bS_{in_i} \bA_i)/n_i-\log|\bA_i|,
\label{1.2}
\end{align}
where $\bA_{i}$ is a positive definite matrix satisfying the equation that $\tr\{\bSig_i (\bA_{i'}-\bA_i) \}=\tr(\bA_i^{-1}\bA_{i'})-p$ $(i\neq i')$.
Here, $\tr(\bS_{in_i} \bA_i)/n_i$ is a bias correction term. 
We consider a quadratic classification rule in which one classifies the individual into $\pi_1$ if
\begin{align}
W_1(\bA_1)-W_2(\bA_2)<0
\label{1.3}
\end{align}
and into $\pi_2$ otherwise. 
Note that (\ref{1.3}) becomes a linear classifier when $\bA_1=\bA_2$.
We have that $E\{W_{i'}(\bA_{i'}) \}-E\{W_i(\bA_i) \}=\Delta_{i}$ when $\bx_0 \in \pi_i$, where
\begin{align}
\Delta_{i}=\bmu_{12}^T\bA_{i'}\bmu_{12}+\tr\{\bSig_i (\bA_{i'}-\bA_i) \}+\log |\bA_{i'}^{-1}\bA_i|
\label{1.4}
\end{align} 
for $i=1,2\ (i' \neq i)$.
\begin{pro}
\label{pro1.1}
(i)\ $\Delta_{i} \ge 0$. 
(ii)\ $\Delta_{i}>0$ when $\bmu_1\neq \bmu_2$ or $\bA_1\neq \bA_2$. 
\end{pro}
\begin{rem}
\label{rem1.1}
As for $l\ (\ge 3)$-class classification, one may consider a classification rule such as one classifies the individual into $\pi_i$ if
$$
\argmin_{i'=1,...,l}W_{i'}(\bA_{i'})=i.
$$
\end{rem}
In this paper, we specially consider the following four typical $\bA_i$s in (\ref{1.2}):
\begin{description}
 \item[(I)]  $\bA_i=\bI_p$, \ \ {\bf (II)} $\displaystyle \bA_i=\frac{p}{\tr(\bSig_i)} \bI_p$, \ \ {\bf (III)}  $\bA_i=\bSig_{i(d)}^{-1}$, \ and \ \ {\bf (IV)} $\bA_i=\bSig_{i}^{-1} $,
\end{description}
where $\bSig_{i(d)}=\mbox{diag}(\sigma_{i(1)},...,\sigma_{i(p)})$.
These four $\bA_i$s satisfy the condition that $\tr\{\bSig_i (\bA_{i'}-\bA_i) \}=\tr(\bA_i^{-1}\bA_{i'})-p$ $(i\neq i')$ and they provide historical background of discriminant analysis.
Note that $|| \bSig_{12}||_F\ge ||\bA_1^{-1}-\bA_2^{-1} ||_F$ for these four $\bA_i$s. 
Also, under (I) to (IV), we note that $\Delta_i\to\infty$ as $p\to\infty$ when $\bmu_{12}$ or $\bSig_{12}$ is non-sparse.
Practically, $\bA_i$s should be estimated except for (I). 
We will consider quadratic classifiers given by estimating $\bA_i$s in Section 4.
Now, let us see an easy example to check the performance of (I) to (IV) in (\ref{1.3}). 
We set $p=2^{s},\ s=3,...,12$. 
Independent pseudo random observations were generated from $\pi_i: N_p(\bmu_i, \bSigma_i)$, $i=1,2$.
We set $\bmu_1=\bze$ and $\bSig_1=\bB_1( 0.3^{|i-j|^{1/3}})\bB_1$, where
$
\bB_1=\mbox{diag}[\{0.5+1/(p+1)\}^{1/2},...,\{0.5+p/(p+1)\}^{1/2}]. 
$
Note that $\tr(\bSig_1)=p$ and $\bSig_{1(d)}=\bB_1^2$. 
When $\bSig_1=\bSig_2$ and $(n_1,n_2)=(\log_2{p},2\log_2{p})$, we considered two cases: 
\begin{flushleft}
(a) $\bmu_{2}=(1,...,1,0,...,0)^T$ whose first $\lceil p^{2/3} \rceil $ elements are $1$, and\\
(b) $\bmu_{2}=(0,...,0,1,...,1)^T$ whose last $\lceil p^{2/3} \rceil $ elements are $1$.
\end{flushleft}
Here, $\lceil x \rceil$ denotes the smallest integer $\ge x$. 
Next, when $\bmu_{2}=\bze$ (i.e., $\bmu_{12}=\bze$) and $(n_1,n_2)=(5,10)$, we considered two cases: 
\begin{flushleft}
(c) $\bSig_2=1.5\bSig_1$ \ and \ (d) $\bSig_2=1.2\bI_p$.
\end{flushleft}
Note that $\bmu_{12}$ or $\bSig_{12}$ is non-sparse for (a) to (d) because $||\bmu_{12}|| \to \infty$ or $||\bSig_{12}||_F\to \infty$ as $p\to \infty$. 
For $\bx_0\in\pi_i\ (i=1,2)$ we repeated 2000 times to confirm if the classification rule by (\ref{1.3}) with either of (I) to (IV) does (or does not) classify $\bx_0$ correctly and defined $P_{ir}=0\ (\mbox{or}\ 1)$ accordingly for each $\pi_i$.
We calculated the error rates, $\overline{e}(i)= \sum_{r=1}^{2000}P_{ir}/2000$, $i=1,2$. 
Also, we calculated the average error rate, $\overline{e}=\{\overline{e}(1)+\overline{e}(2)\}/2$.
Their standard deviations are less than $0.011$. 
In Figure~\ref{F1}, we plotted $\overline{e}$ for (a) and (b). 
Note that (I) is equivalent to (II) for (a) and (b). 
In Figure~\ref{F2}, we plotted $\overline{e}$ for (c) and (d). 
\begin{figure}
\begin{centering}
\includegraphics[scale=0.385]{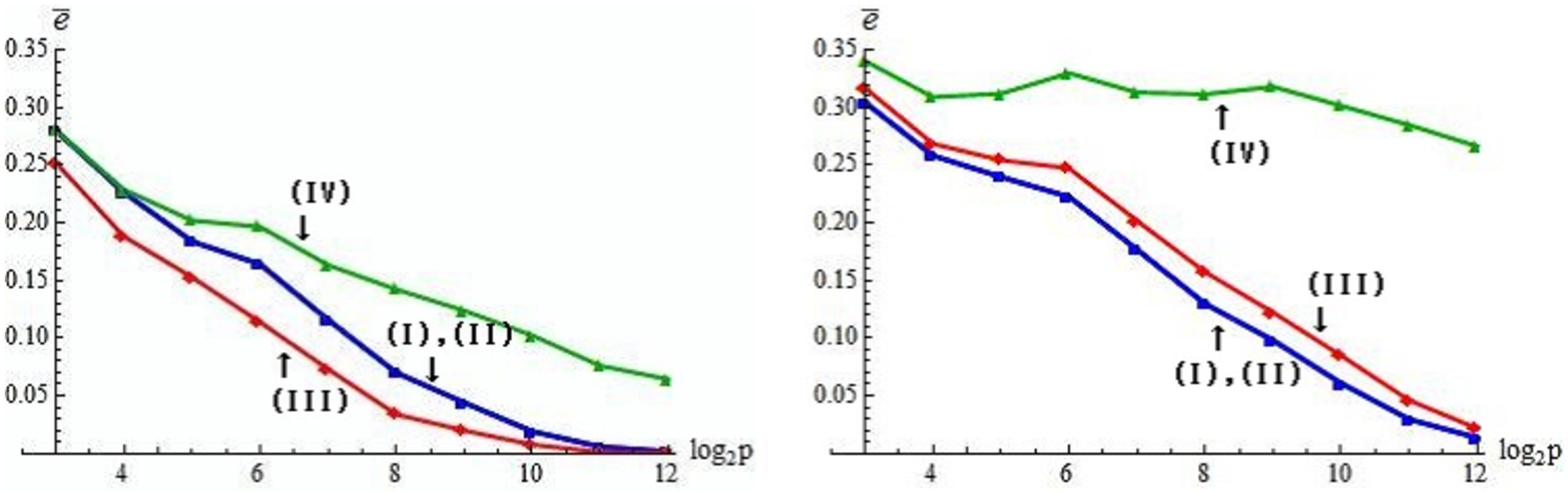} 
\\[-1mm]
\ \ 
(a) $\bmu_1=\bze$ and $\bmu_{2}=(1,...,1,0,...,0)^T$ \hspace{0.1cm} (b) $\bmu_1=\bze$ and $\bmu_{2}=(0,...,0,1,...,1)^T$
\\[-2.5mm]
\caption{The average error rates of the classification rule by (\ref{1.3}) for (I) to (IV) when $\bSig_1=\bSig_2$. 
The left and right panels display $\overline{e}$ in the cases of (a) and (b), respectively.}
\label{F1}
\includegraphics[scale=0.385]{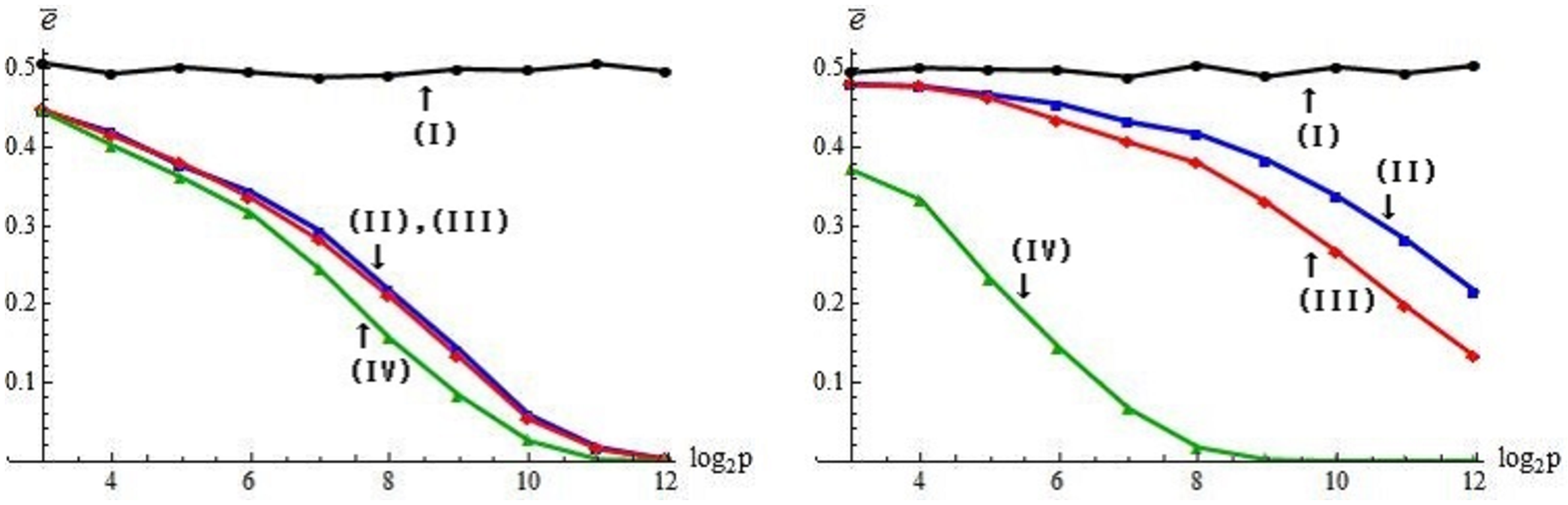} 
\\[-1mm]
(c) $\bSig_2=1.5\bSig_1$ \hspace{3cm} (d) $\bSig_2=1.2\bI_p$
\\[-3mm]
\caption{The average error rates of the classification rule by (\ref{1.3}) for (I) to (IV) when $\bmu_1=\bmu_2$.
The left and right panels display $\overline{e}$ in the cases of (c) and (d), respectively.}
\label{F2}
\end{centering}
\end{figure}
We observed that (IV) gives the worst performance in Figure~\ref{F1} contrary to expectations. 
In general, one would think that the classifier based on the Mahalanobis distance such as (\ref{1.2}) with (IV) is the best when $\pi_i$s are Gaussian and $n_{\min}\to \infty$. 
We emphasize that it is not true for high-dimensional data. 
We will explain its theoretical reason in Section 3.2. 
We observed that (I) (or (II)) gives a better performance compared to (III) for (b) in Figure~\ref{F1}. 
We will discuss the reasons in Section 3.4. 
In Figure~\ref{F2}, the error rates of (I) are close to $0.5$ because of $\bmu_{12}=\bze$. 
On the other hand, (II), (III) and (IV) gave good performances as $p$ increases by drawing information on heteroscedasticity in the classifiers. 
We will give their theoretical backgrounds in Sections 2.2 and 3.4.

We pay special attention to the difference of covariance matrices in classification for high-dimensional data. 
In Section 2, we show that the classification rule by (\ref{1.3}) holds the consistency property under non-sparse settings. 
In Section 3, we verify that the quadratic classifier by (\ref{1.2}) is asymptotically distributed as a normal distribution under certain conditions. 
In Section 4, we consider the estimation of $\bA_i$s and give asymptotic properties of estimated classifiers. 
In Section 5, we propose a quadratic classifier after feature selection by using both the differences of mean vectors and covariance matrices.
In Section 6, we discuss performances of the classifiers in actual data analyses.
Finally, in Section 7, we give concluding remarks of our study. 
\section{Consistency property of the quadratic classifier}
In this section, we discuss the consistency property of quadratic classifiers given by (\ref{1.2}). 
\subsection{Preliminary}
Similar to \cite{Bai:1996} and \cite{Aoshima:2014}, we assume the following assumption about population distributions as necessary:
\begin{description}
\item[(A-i)]\
Let $\by_{ik},\ k=1,...,n_i$, be i.i.d. random $q_i$-vectors having $E(\by_{ik})=\bze$ and $\Var(\by_{ik})=\bI_{q_i}$ for each $i\ (=1,2)$, where $q_i\ge p$.
Let $\by_{ik}=(y_{i1k},...,y_{iq_ik})^T$ whose components satisfy that $\limsup_{p\to \infty} E(y_{ijk}^4)<\infty$ for all $j$ and
\begin{equation}
E( y_{ijk}^2y_{irk}^2)=E( y_{ijk}^2)E(y_{irk}^2)=1\quad \mbox{and} 
\quad  E( y_{ijk} y_{irk}y_{isk}y_{itk})=0
\label{2.1}
\end{equation}
for all $j\neq r,s,t$. 
Then, the observations, $\bx_{ik}$s, from each $\pi_i\ (i=1,2)$ are given by
$$
\bx_{ik}=\bGamma_i\by_{ik}+\bmu_i,\ k=1,...,n_i, 
$$
where $\bGamma_i=[\bgam_{i1},...,\bgam_{iq_i}]$ is a $p\times q_i$ matrix such that $\bGamma_i\bGamma_i^T=\bSigma_i$. 
\end{description} 
Note that $\bGamma_i$ includes the case that $\bGamma_i=\bH_i\bLam_i^{1/2}=[\lambda_{i1}^{1/2}\bh_{i1},...,\lambda_{ip}^{1/2}\bh_{ip}]$. 
We assume the following assumption instead of (A-i) as necessary:
\begin{description}
\item[(A-ii)]\ (A-i) by replacing (\ref{2.1}) with the independence of $y_{ijk},\ j=1,...,q_i\ (i=1,2;\ k=1,...,n_i)$.
\end{description} 
Note that (A-ii) is a special case of (A-i). 
When $\pi_i$ has $N_p(\bmu_i,\bSig_i)$, (A-ii) naturally holds.

Now, we consider the following divergence condition for $p$ and $n_i$s: 
\begin{description}
\item[($\star$)]\ $p\to \infty$ either when $n_i$ is fixed or $n_i\to \infty$ for $i=1,2$.
\end{description} 
Let $\Delta_{iA}=\bmu_{12}^T\bA_{i'} \bSig_i \bA_{i'} \bmu_{12}$ for $i=1,2\ (i' \neq i)$. 
We consider the following conditions under ($\star$) for $i=1,2$ $(i' \neq i)$: 
\begin{description}
 \item[(C-i)]\ $\displaystyle \frac{\tr\{(\bSig_i\bA_{i})^2\} }{n_i \Delta_{i}^2 }=o(1) $  \ and \  
 $\displaystyle \frac{\tr(\bSig_i \bA_{i'} \bSig_{i'}  \bA_{i'})+\tr\{(\bSig_{i'} \bA_{i'})^2\}/n_{i'} }{n_{i'}\Delta_{i}^2 }=o(1)$, 
 \item[(C-ii)]\ $\displaystyle \frac{\Delta_{iA}}{\Delta_{i}^2 }=o(1)$, \ and \ 
 {\bf (C-iii)}\ $\displaystyle \frac{\tr[\{\bSig_i(\bA_{1}-\bA_{2})\}^2] }{\Delta_{i}^2 }=o(1)$. 
\end{description}
\par
Then, we claim the consistency property of (\ref{1.2}) in (\ref{1.3}) as follows:
\begin{thm}
\label{thm2.1}
Assume (A-i). 
Assume also (C-i) to (C-iii). 
Then, we have that 
$$
\frac{W_{i'}(\bA_{i'})-W_i(\bA_i)}{\Delta_{i}}=1+o_P(1) \ \ \mbox{under ($\star$) when $\bx_0 \in \pi_i$ for $i=1,2$ $(i' \neq i)$}.
$$ 
Furthermore, for the classification rule by (\ref{1.3}) with (\ref{1.2}), we have that 
\begin{align}
e(i)\to 0,\ i=1,2,\ \mbox{under ($\star$).} 
\label{2.2}
\end{align}
\end{thm}
\begin{rem}
When $\bA_1=\bA_2$, we can claim Theorem \ref{thm2.1} without (A-i) and (C-iii). 
\end{rem}
Let $\lambda_{\min}(\bM)$ and $\lambda_{\max}(\bM)$ be the smallest and the largest eigenvalues of any positive definite matrix, $\bM$. 
We use the phrase ``$\lambda(\bM)\in (0,\infty)$ as $p\to \infty$" in the sense that $\liminf_{p\to \infty}\lambda_{\min}(\bM)>0$ and $\limsup_{p\to \infty}\lambda_{\max}(\bM)<\infty$.
We note that $\bA_i$s in (I) to (III) satisfy the condition ``$\lambda(\bA_i)\in (0,\infty)$ as $p\to \infty$".
Let $\Delta_{\min}=\min\{\Delta_{1},\Delta_{2} \}$, $\lambda_{\max}=\max\{\lambda_{\max}(\bSig_1),\lambda_{\max}(\bSig_2) \}$ and 
$\tr(\bSig_{\max}^2)=\max\{\tr(\bSig_1^2), \tr(\bSig_2^2) \}$. 
Now, instead of (C-i) and (C-ii), we consider the following simpler conditions under ($\star$):
\begin{description}
\item[(C-i')]\ $\displaystyle \frac{\tr(\bSig_{\max}^2)}{n_{\min} \Delta_{\min}^2}=o(1)$ \ and \ 
{\bf (C-ii')}\ $\displaystyle \frac{\lambda_{\max}}{\Delta_{\min}}=o(1)$. 
\end{description}
\begin{pro}
\label{pro2.1}
Assume that $\limsup_{p\to \infty} \lambda_{\max}(\bA_i)<\infty$ for $i=1,2$.
Then, (C-i') and (C-ii') imply (C-i) and (C-ii), respectively. 
Furthermore, if $\lambda(\bA_i)\in (0,\infty)$ as $p\to \infty$ for $i=1,2$, and $\bA_i,\ i=1,2,$ are diagonal matrices such as in (I) to (III) in Section 1, (C-ii') implies (C-iii).
\end{pro}

From the fact that $\lambda_{\max}(\bSig_i) \le  \tr(\bSig_i^2)^{1/2}$ for $i=1,2$, we note that (C-i') and (C-ii') hold even when $n_{\min}$ is fixed under 
\begin{align}
\tr(\bSig_{\max}^2)/\Delta_{\min}^2
\to 0 \ \  
\mbox{as $p\to \infty$.} \label{2.3}
\end{align}
\subsection{Consistency property for (I) to (IV)}
As mentioned in Section 1, four typical $\bA_i$s were specifically selected.
For (I), by putting $\bA_i=\bI_p, i=1,2$, (\ref{1.2}) and (\ref{1.4}) are given as
\begin{align}
&W_i(\bI_p)=||\bx_0-\overline{\bx}_{in_i}||^2-\tr(\bS_{in_i})/n_i 
\label{2.4} \\
&\mbox{and} \ \Delta_{1}=\Delta_{2}=||\bmu_{12}||^2 \ \ (\mbox{hereafter called }\Delta_{(I)}). \notag
\end{align}
For (II), by putting $\bA_i=\{p/\tr(\bSig_i)\} \bI_p, i=1,2$, they are given as
\begin{align}
&W_i(\{p/\tr(\bSig_i)\} \bI_p)=\frac{p||\bx_0-\overline{\bx}_{in_i}||^2}{\tr(\bSig_i)}-\frac{p\tr(\bS_{in_i})}{n_i\tr(\bSig_i)}+p\log\{\tr(\bSig_i)/p \} 
\label{2.5} \\ 
&\mbox{and}\ \Delta_{i}=\frac{p\Delta_{(I)}}{\tr(\bSig_{i'})}+\frac{p \tr(\bSig_i)}{\tr(\bSig_{i'})}-p+p\log\Big\{ \frac{\tr(\bSig_{i'})}{\tr(\bSig_i)} \Big\} \ \ (\mbox{hereafter called }\Delta_{i(II)}). \notag
\end{align}
For (III), by putting $\bA_i= \bSig_{i(d)}^{-1}, i=1,2$, they are given as
\begin{align}
&W_i(\bSig_{i(d)}^{-1})=\sum_{j=1}^p\Big(\frac{(x_{0j}-\overline{x}_{ijn_i})^2}{\sigma_{i(j)}}-\frac{s_{in_i(j)} }{n_i\sigma_{i(j)}}+\log{\sigma_{i(j)} }\Big)
\label{2.6} \\
&\mbox{and}\ \Delta_{i}=\sum_{j=1}^p\Big\{ \frac{\mu_{12j}^2}{\sigma_{i'(j)}}+ \frac{\sigma_{i(j)}}{\sigma_{i'(j)}}-1+\log\Big(\frac{\sigma_{i'(j)}}{\sigma_{i(j)}} \Big) \Big\} \ \ (\mbox{hereafter called }\Delta_{i(III)}). \notag
\end{align} 
For (IV), by putting $\bA_i=\bSig_{i}^{-1}, i=1,2$, they are given as 
\begin{align}
&W_i(\bSig_{i}^{-1})=(\bx_0-\overline{\bx}_{in_i})^T\bSig_i^{-1}(\bx_0-\overline{\bx}_{in_i})-\frac{\tr(\bS_{in_i}\bSig_i^{-1})}{n_i}+\sum_{j=1}^p\log{\lambda_{ij}}
\label{2.7} \\
&\mbox{and}\ \Delta_{i}= \bmu_{12}^T\bSig_{i'}^{-1}\bmu_{12}+\tr(\bSig_i \bSig_{i'}^{-1})-p+\sum_{j=1}^p\log\Big(\frac{\lambda_{i'j}}{\lambda_{ij}} \Big) \ \
(\mbox{hereafter called }\Delta_{i(IV)}). \notag
\end{align}
We first consider the classifiers by (\ref{2.4}) to (\ref{2.6}). 
From Theorem \ref{thm2.1} and Proposition \ref{pro2.1}, we have the following result. 
\begin{cor}
\label{cor2.1}
Assume (C-i') and (C-ii'). 
Then, for the classification rule by (\ref{1.3}) with (\ref{2.4}), we have (\ref{2.2}). 
Furthermore, for the classification rule by (\ref{1.3}) with (\ref{2.5}) or (\ref{2.6}), we have (\ref{2.2}) under (A-i).
\end{cor}
We note that the classifier by (\ref{2.4}) is equivalent to the distance-based classifier by \cite{Aoshima:2014}.
Hereafter, we call the classifier by (\ref{2.4}) the ``distance-based discriminant analysis (DBDA)".
From Corollary \ref{cor2.1}, under (\ref{2.3}), the classification rule by (\ref{1.3}) with (\ref{2.4}), (\ref{2.5}) or (\ref{2.6}) has (\ref{2.2}) even when $n_i$s are fixed.
Note that DBDA has the consistency property without (A-i), so that DBDA is quite robust for non-Gaussian cases.
See \cite{Aoshima:2014} for details. 
When $\bmu_1=\bmu_2$, DBDA does not satisfy (C-i') and (C-ii'), on the other hand, the classifier by (\ref{2.5}) or (\ref{2.6}) still satisfies them.

Now, we consider the following condition for $\bSig_i, i=1,2$: 
\begin{equation}
\tr(\bSig_i^2)/\tr(\bSig_i)^2
\to 0\ \mbox{as $p\to \infty$}. 
\label{2.8}
\end{equation}
We note that $\tr(\bSig_i^2)/\tr(\bSig_i)^2$ is a measure of sphericity. 
Also, note that (\ref{2.8}) is equivalent to the condition that ``$\lambda_{\max}(\bSig_i)/\tr(\bSigma_i)\to 0$ as $p\to \infty$". 
Under (A-i) and (\ref{2.8}), from the fact that $\Var(||\bx_0-\bmu_i||^2)=O\{\tr(\bSig_i^2)\}$ when $\bx_0 \in \pi_i$, we have that as $p\to \infty$ 
$$
||\bx_0-\bmu_i||=\tr(\bSig_i)^{1/2}\{1+o_P(1)\} \ \mbox{when $\bx_0 \in \pi_i$}.
$$
Thus the centroid data lies near the surface of an expanding sphere. 
See \cite{Hall:2005} for details of the geometric representation. 
We emphasize that the classifier by (\ref{2.5}) draws information about heteroscedasticity thorough the geometric representation having different radii, $\tr(\bSig_i)^{1/2}$s, of expanding two spheres. 
Note that $\tr(\bSig_i^2)=o(p^2)$ under (\ref{2.8}). 
Hence, for the classifier by (\ref{2.5}), (\ref{2.3}) holds under (\ref{2.8}) and $\liminf_{p\to \infty} \Delta_{\min(II)}/p>0$, where $\Delta_{\min(II)}=\min\{\Delta_{1(II)},\Delta_{2(II)} \}$. 
Note that $\Delta_{\min(II)}>0$ when $\tr(\bSig_1)\neq \tr(\bSig_2)$ in view of Proposition \ref{pro1.1}. 
If one can assume that $\liminf_{p\to \infty}|\tr(\bSig_1)/\tr(\bSig_2)-1|>0$, it follows $\liminf_{p\to \infty} \Delta_{\min(II)}$
$/p>0$, so that (\ref{2.3}) holds under (\ref{2.8}).
Hence, for the classification rule by (\ref{1.3}) with (\ref{2.5}), we have (\ref{2.2}) even when 
$\bmu_1=\bmu_2$ and $n_i$s are fixed.
See (II) in Figure~\ref{F2}.
The accuracy becomes higher as the difference between $\tr(\bSig_i)$s grows. 

Similarly, for the classifier by (\ref{2.6}), it follows that (\ref{2.3}) holds under (\ref{2.8}) and $\liminf_{p\to \infty}$
$ \Delta_{\min(III)}/p>0$, where $\Delta_{\min(III)}=\min\{\Delta_{1(III)},\Delta_{2(III)} \}$. 
If one can assume that $\liminf_{p\to \infty}$
$\sum_{j=1}^p|\sigma_{1(j)}/\sigma_{2(j)}-1|/p>0$, it follows $\liminf_{p\to \infty} \Delta_{\min(III)}/p>0$, so that the classification rule by (\ref{1.3}) with (\ref{2.6}) has (\ref{2.2}) even when $\bmu_1=\bmu_2$ and $n_i$s are fixed. 
The classifier by (\ref{2.6}) draws information about heteroscedasticity via the difference of diagonal elements between the two covariance matrices. 
The accuracy becomes higher as the difference of those diagonal elements grows.
See (III) in Figure~\ref{F2}. 

Next, we consider the classifier by (\ref{2.7}). 
From Theorem \ref{thm2.1} and Proposition \ref{pro2.1}, we have the following result. 
\begin{cor}
\label{cor2.2}
Assume (A-i).
Assume also $\liminf_{p\to \infty}\lambda_{\min}(\bSig_i)>0$ for $i=1,2$. 
Then, for the classification rule by (\ref{1.3}) with (\ref{2.7}), we have (\ref{2.2}) under (C-i'), (C-ii') and the condition that $\tr\{(\bI_p-\bSig_i\bSig_{i'}^{-1})^2\}=o(\Delta_{\min(IV)}^2)$ for $i=1,2\ (i'\neq i)$, where $\Delta_{\min(IV)}=\min\{\Delta_{1(IV)},\Delta_{2(IV)} \}$.
\end{cor}
When $\bSig_1 \neq \bSig_2$, note that $\Delta_{\min(IV)}>0$ in view of Proposition \ref{pro1.1}.
Then, we have the following result.
\begin{pro}
\label{pro2.2}
When $\liminf_{p\to \infty}|\tr(\bSig_i\bSig_{i'}^{-1})/p-1|>0$ or $\liminf_{p\to \infty}\sum_{j=1}^p|\lambda_{ij}/\lambda_{i'j}$
$-1|/p>0$ $(i\neq i')$, it follows that $\liminf_{p\to \infty} \Delta_{i(IV)}/p>0$.
\end{pro}
Note that $\tr\{(\bI_p-\bSig_i\bSig_{i'}^{-1})^2\}\le p+\tr\{(\bSig_i\bSig_{i'}^{-1})^2\}=p+O\{\tr(\bSig_i^2) \}=o(p^2)$ under (\ref{2.8}) and $\liminf_{p\to \infty}\lambda_{\min}(\bSig_{i'}) >0$. 
Hence, from Corollary \ref{cor2.2}, for the classification rule by (\ref{1.3}) with (\ref{2.7}), we have (\ref{2.2}) under (A-i), (\ref{2.8}), $\liminf_{p\to \infty} \Delta_{\min(IV)}/p>0$ and $\liminf_{p\to \infty}\lambda_{\min}(\bSig_i)>0$ for $i=1,2$.
Thus from Proposition \ref{pro2.2}, the accuracy becomes higher as the difference of eigenvalues or eigenvectors between the two covariance matrices grows.
See (IV) in Figure~\ref{F2}. 
\section{Asymptotic normality of the quadratic classifier}
In this section, we discuss the asymptotic normality of quadratic classifiers given by (\ref{1.2}). 
We further discuss Bayes error rates for high-dimensional data.
\subsection{Preliminary}
Let 
$$
\delta_{i}=2
\Big\{ \frac{\tr\{(\bSig_i\bA_{i} )^2 \}}{n_i}+\frac{\tr(\bSig_i\bA_{i'} \bSig_{i'} \bA_{i'})}{n_{i'}}
+\Delta_{iA} \Big\}^{1/2} \ \mbox{for $i=1,2\ (i'\neq i)$.}
$$
Note that 
$\delta_{i}^2=
\Var[ 2(\bx_0-\bmu_i)^T \{\bA_i(\overline{\bx}_{in_i}-\bmu_i)-\bA_{i'}(\overline{\bx}_{i'n_{i'}}-\bmu_{i'}+
(-1)^i\bmu_{12})\} ]$ for $i=1,2\ (i' \neq i)$.
Let $m=\min\{p,n_{\min} \}$. 
We assume the following conditions when $m\to \infty$ for $i=1,2\ (i' \neq i)$: 
\begin{description}
 \item[(C-iv)]\ $\displaystyle \frac{\bmu_{12}^T\bA_{i'} \bSig_{i'} \bA_{i'}\bmu_{12}+\tr\{(\bSig_{i'}\bA_{i'})^2\}/n_{i'} }{n_{i'} \delta_{i}^2 }=o(1)$, \ \ 
$\displaystyle \frac{\tr\{ (\bSig_i\bA_{i})^4 \} }{n_i^2 \delta_{i}^4 }=o(1) $ \ and \\ 
$\displaystyle \frac{ \tr\{(\bSig_i\bA_{i'} \bSig_{i'}\bA_{i'})^2 \}}{n_{i'}^2 \delta_{i}^4 }=o(1) $;
 \item[(C-v)]\ $\displaystyle \frac{\tr[\{\bSig_i(\bA_{1}-\bA_{2})\}^2] }{\delta_{i}^2 }= o(1)$; 
  \ and \ {\bf (C-vi)}\ $\displaystyle \frac{\Delta_{iA}}{\delta_{i}^2 }=o(1)$. 
\end{description}
From (\ref{A.6}) in Appendix, under (A-i), (C-iv) and (C-v), it holds that 
\begin{align}
W_{i'}(\bA_{i'})-W_i(\bA_i)-\Delta_{i} =&
2(\bx_0-\bmu_i)^T\Big\{\bA_i(\overline{\bx}_{in_i}-\bmu_i) \notag \\  
&-\bA_{i'}(\overline{\bx}_{i'n_{i'}}-\bmu_{i'}+(-1)^i\bmu_{12}) \Big\}+o_P(\delta_{i}) \notag
\end{align}
as $m\to \infty$ when $\bx_0\in\pi_i$ for $i=1,2\ (i'\neq i)$.
Under (C-vi), it holds that 
$(\bx_0-\bmu_i)^T\bA_{i'}\bmu_{12}=o_P(\delta_{i})$ as $m\to \infty$ when $\bx_0\in\pi_i$ for $i=1,2\ (i' \neq i)$.
Then, we claim the asymptotic normality of (\ref{1.2}) under (A-i) as follows:
\begin{thm}
\label{thm3.1}
Assume (A-i). 
Assume also (C-iv) to (C-vi). 
Then, we have that 
\begin{align}
&\frac{W_{i'}(\bA_{i'})-W_i(\bA_i)-\Delta_{i}}{\delta_{i} }\Rightarrow N(0,1)\ \ \mbox{as $m\to \infty$}  \label{3.1}\\
&\mbox{when $\bx_0\in\pi_i$ for $i=1,2\ (i' \neq i)$}, \notag
\end{align}
where ``$\Rightarrow$" denotes the convergence in distribution and $N(0,1)$ denotes a random variable distributed as the standard normal distribution.
Furthermore, for the classification rule by (\ref{1.3}) with (\ref{1.2}), it holds that 
\begin{align}
\label{3.2}
e(i)= \Phi\Big( \frac{-\Delta_{i}}{\delta_{i}} \Big)+o(1) \ \ \mbox{as $m\to \infty$ for $i=1,2$}.
\end{align}
\end{thm}
Let $\delta_{\min}=\min\{\delta_{1},\delta_{2} \}$. 
Now, instead of (C-iv) to (C-vi), we consider the following conditions when $m\to \infty$: 
\begin{description}
  \item[(C-iv')]\ $\displaystyle \frac{ ||\bmu_{12}||^2\lambda_{\max}+\tr(\bSig_{\max}^2)/n_{\min}}{n_{\min} \delta_{\min}^2 }=o(1)$ \ and \  
$\displaystyle \frac{ \lambda_{\max}^2 }{n_{\min} \delta_{\min}^2}=o(1)$,
 \item[(C-v')]\ 
 $\displaystyle \frac{ \tr\{(\bA_{1}-\bA_2)^2\} \lambda_{\max} }{\delta_{\min}^2}=o(1)$, 
 \ and \  {\bf (C-vi')}\ $\displaystyle \frac{||\bmu_{12}||^2 \lambda_{\max}}{ \delta_{\min}^2 }=o(1)$.
\end{description}
\begin{pro}
\label{pro3.1}
Assume that $\limsup_{p\to \infty} \lambda_{\max}(\bA_i)<\infty$ for $i=1,2$.
Then, (C-iv') and (C-vi') imply (C-iv) and (C-vi), respectively. 
Furthermore, if $\lambda(\bA_i)\in (0,\infty)$ as $p\to \infty$ for $i=1,2$, and $\bA_i,\ i=1,2,$ are diagonal matrices such as in (I) to (III) in Section 1, (C-v') implies (C-v).
\end{pro}

Next, we consider the asymptotic normality of (\ref{1.2}) under (A-ii). 
We assume the following condition instead of (C-vi) when $m\to \infty$ for $i=1,2\ (i'\neq i)$: 
\begin{description}
  \item[(C-vii)]\  $\displaystyle \frac{\sum_{j=1}^{q_i} (\bgam_{ij}^T \bA_{i'} \bmu_{12})^4 }{ \delta_{i}^4 }=o(1)$. 
\end{description}
Note that 
$\sum_{j=1}^{q_i} (\bgam_{is}^T \bA_{i'}\bmu_{12})^4\le \sum_{j,j'=1}^{q_i} (\bgam_{ij}^T \bA_{i'}\bmu_{12})^2(\bgam_{ij'}^T \bA_{i'}\bmu_{12} )^2=\Delta_{iA}^2$. 
Thus (C-vii) is milder than (C-vi). 
\begin{rem}
The condition in (C-vii) can be written as a condition concerning eigenvalues and eigenvectors.
If $\bGamma_i=\bH_i \bLam_i^{1/2}$, $\bA_i=\bSig_i^{-1},\ i=1,2$, and $\bSig_1=\bSig_2$, it holds that $\sum_{j=1}^{q_i} \{\bgam_{ij}^T \bA_{i'}\bmu_{12}\}^4=\sum_{j=1}^{p} \psi_j^2$ and $\Delta_{i A}=\bmu_{12}^T\bSig_i^{-1}\bmu_{12}=\sum_{j=1}^{p}\psi_j$, where $\psi_j=(\bmu_{12}^T\bh_{ij})^2$
$/\lambda_{ij}$.
Hence, the condition `` $\sum_{j=1}^{p} \psi_j^2/(\sum_{j=1}^{p} \psi_j)^2\to 0$ as $p\to \infty$" implies (C-vii). 
\end{rem}
Now, we claim the asymptotic normality of (\ref{1.2}) under (A-ii) as follows:
\begin{thm}
\label{thm3.2}
Assume (A-ii). 
Assume also (C-iv), (C-v) and (C-vii). 
Then, we have (\ref{3.1}). 
Furthermore, for the classification rule by (\ref{1.3}) with (\ref{1.2}), we have (\ref{3.2}). 
\end{thm}
\subsection{Bayes error rates}
When considering Theorem \ref{thm3.2} under the situation that 
\begin{equation}
\tr\{(\bSig_i\bA_{i})^2 \}/n_i+\tr(\bSig_i\bA_{i'}\bSig_{i'}\bA_{i'})/n_{i'}=o(\Delta_{iA}) \ \ \mbox{as $m\to \infty$} 
\label{3.3}
\end{equation}
for $i=1,2$ ($i' \neq i$), one has (\ref{3.2}) as
$$
e(i)= \Phi\{ -\Delta_{i}/(2\Delta_{iA}^{1/2} )\}+o(1) \ \mbox{as $m\to\infty$ for $i=1,2$}.
$$
Note that $\delta_{i}/(2\Delta_{iA}^{1/2})=1+o(1)$ under (\ref{3.3}). 
If $\bSig_1=\bSig_2(=\bSig)$, the ratio $\Delta_{i}/\Delta_{iA}^{1/2}$ has a maximum when $\bA_1=\bA_2=\bSig^{-1}$. 
Then, the ratio becomes the Mahalanobis distance such as $\Delta_{i}/\Delta_{iA}^{1/2}=\Delta_{MD}^{1/2}$.
The classification rule by (\ref{1.3}) with (\ref{1.2}) has an error rate converging to the Bayes error rate in the sense that $e(i)=\Phi(-\Delta_{MD}^{1/2}/2)+o(1)$ for $i=1,2$. 
On the other hand, if $\bSig_1\neq \bSig_2$ and $\pi_i$s are Gaussian, under (C-iii) for (IV), the Bayes optimal classifier by (\ref{1.1}) becomes as follows: 
$$
 2(\bx_0-\bmu_i)^T\bSig_{i'}^{-1}\bmu_{12}+o_P(\Delta_{i(IV)})>(-1)^{i}\Delta_{i(IV)}
$$
when $\bx_0 \in \pi_i$ ($i'\neq i$). 
Note that $\Var\{(\bx_0-\bmu_i)^T\bSig_{i'}^{-1}\bmu_{12}\}=\bmu_{12}^T\bSig_{i'}^{-1}\bSig_i\bSig_{i'}^{-1}\bmu_{12}$ (hereafter called $\Delta_{iA(IV)}$) when $\bx_0 \in \pi_i$ ($i' \neq i$) and $\Delta_{iA(IV)}$ is the same as $\Delta_{iA}$ for (IV). 
Hence, $(\bx_0-\bmu_i)^T\bSig_{i'}^{-1}\bmu_{12}/\Delta_{iA(IV)}^{1/2}$ is distributed as $N(0,1)$ when $\bx_0 \in \pi_i:N_p(\bmu_i,\bSig_i)$.
Then, the Bayes error rate becomes $e(i)= \Phi\{ -\Delta_{i(IV)}/(2\Delta_{iA(IV)}^{1/2})\}+o(1)$ for $i=1,2$, under some conditions. 

When considering Theorem \ref{thm3.2} under the situation that 
\begin{equation}
p/n_i+\tr(\bSig_i\bSig_{i'}^{-1} )/n_{i'}=o(\Delta_{iA(IV)}) \ \ \mbox{as $m\to \infty$} 
\label{3.4}
\end{equation}
for $i=1,2$ ($i' \neq i$), one can claim that the classification rule by (\ref{1.3}) with (\ref{2.7}) has the Bayes error rate asymptotically even when $\pi_i$s are non-Gaussian. 
Note that (\ref{3.4}) is equivalent to (\ref{3.3}) for (IV) and (\ref{3.4}) usually holds when $n_{\min}\to \infty$ while $p$ is fixed or $p\to \infty$ but $n_{\min}/p\to \infty$.
If (\ref{3.4}) is not met, the classifier by (\ref{2.7}) is not optimal. 
We emphasize that (\ref{3.4}) does not always hold for high-dimensional settings such as $n_{\min}/p\to 0$ or $n_{\min}/p\to c\ (>0)$. 
For example, let us consider the setup of Figure~\ref{F1}. 
The condition ``$p/n_i=o(\Delta_{iA(IV)})$" is not met from the facts that $\Delta_{iA(IV)}=O(p^{2/3})$ and $n_1=n_2=o(p^{1/3})$, so that (\ref{3.4}) does not hold.
On the other hand, (C-iv) to (C-vi) hold, so that one can claim the asymptotic normality in Theorem \ref{thm3.1}.
Note that (\ref{3.4}) does not hold under (C-vi) for (IV). 
Thus the error rate of the classifier based on the Mahalanobis distance does not converge to the Bayes error rate when Theorem \ref{thm3.1} is claimed. 
Such situations frequently occur in HDLSS settings such as $n_{\min}/p\to 0$. 
This is the reason why the classifier based on the Mahalanobis distance does not always give a preferable performance for high-dimensional data even when $n_{\min}\to \infty$, $\bSig_i$s are known and $\pi_i$s are Gaussian.
\subsection{Asymptotic normality for (I) to (IV)}
We consider $\delta_{i}$s for (I) to (IV).
For (I), by putting $\bA_i= \bI_p, i=1,2$, one has $\delta_{i} (i\ne i')$ as
$$
\delta_{i}=2\Big\{\frac{\tr(\bSig_i^2)}{n_i}+\frac{\tr(\bSig_i\bSig_{i'})}{n_{i'}}+\bmu_{12}^T\bSig_i\bmu_{12}\Big\}^{1/2} \ (\mbox{hereafter called }\delta_{i(I)}).
$$
For (II), by putting $\bA_i=\{p/\tr(\bSig_i)\} \bI_p, i=1,2$, it is given as
\begin{align*}
\delta_{i}= \frac{2p}{\tr(\bSig_{i'})}\Big\{ \frac{\delta_{i(I)}^2}{4}+ \frac{\tr(\bSig_i^2)}{n_i}\Big(\frac{\tr(\bSig_{i'})^2}{\tr(\bSig_i)^2}-1\Big) \Big\}^{1/2} \ (\mbox{hereafter called }\delta_{i(II)}).
\end{align*} 
For (III), by putting $\bA_i= \bSig_{i(d)}^{-1}, i=1,2$, it is given as 
\begin{align*}
&\delta_{i}=2\Big\{ \frac{\tr\{(\bSig_i\bSig_{i(d)}^{-1})^2\}}{n_i}+\frac{\tr(\bSig_i\bSig_{i'(d)}^{-1} \bSig_{i'}\bSig_{i'(d)}^{-1})}{n_{i'}}+\bmu_{12}^T \bSig_{i'(d)}^{-1} \bSig_i \bSig_{i'(d)}^{-1}\bmu_{12}\Big\}^{1/2} \\ 
&(\mbox{hereafter called }\delta_{i(III)}).
\end{align*} 
For (IV), by putting $\bA_i=\bSig_{i}^{-1}, i=1,2$, it is given as 
$$
\delta_{i}=2\Big\{ \frac{p}{n_i}+ \frac{\tr(\bSig_i \bSig_{i'}^{-1} )}{n_{i'}}+\bmu_{12}^T\bSig_{i'}^{-1} \bSig_i \bSig_{i'}^{-1}\bmu_{12} \Big\}^{1/2}\ (\mbox{hereafter called }\delta_{i(IV)}).
$$
 
From Theorems \ref{thm3.1}, \ref{thm3.2} and Proposition \ref{pro3.1}, we have the following result for (I) to (III). 
\begin{cor}
\label{cor3.1}
Assume (C-iv').
Assume either (A-i) and (C-vi') or (A-ii) and (C-vii). 
Then, for the classification rule by (\ref{1.3}) with (\ref{2.4}), we have (\ref{3.2}). 
Furthermore, under (C-v'), for the classification rule by (\ref{1.3}) with (\ref{2.5}) or (\ref{2.6}), we have (\ref{3.2}).
\end{cor}
\begin{rem}
\label{rem3.1}
When $\tr(\bSig_1)/\tr(\bSig_2)\to 1$ as $p\to \infty$, it holds $\{\delta_{i(I)}p/\tr(\bSig_{i'})\}/\delta_{i(II)}=1+o(1)\ (i\neq i')$. 
Note that $\Delta_{i(II)}\tr(\bSig_{i'})/p\ge \Delta_{(I)}$.
It follows that $\Delta_{(I)}/\delta_{i(I)} \le \Delta_{i(II)}/\delta_{i(II)}$ for sufficiently large $p$ in (\ref{3.2}).
\end{rem}
From Theorems \ref{thm3.1} and \ref{thm3.2} and Proposition \ref{pro3.1}, we have the following result for (IV).
\begin{cor}
\label{cor3.2}
Assume that (C-iv'), $\liminf_{p\to \infty}\lambda_{\min}(\bSig_i)>0$ and $\tr\{(\bI_p-\bSig_i\bSig_{i'}^{-1})^2\}=o(\delta_{\min(IV)}^2)$
for $i=1,2$ $(i'\neq i)$, 
where $\delta_{\min(IV)}=\min\{\delta_{1(IV)}$
$,\delta_{2(IV)}\}$. 
Assume either (A-i) and (C-vi') or (A-ii) and (C-vii). 
Then, for the classification rule by (\ref{1.3}) with (\ref{2.7}), we have (\ref{3.2}). 
\end{cor}
\subsection{Comparisons of the classifiers}
In this section, we investigate the performance of the classifier in (\ref{1.2}) for (I) to (IV) by using the asymptotic normality. 
When $\bSig_1=\bSig_2$, we consider (I), (III) and (IV) in the setup of Figure~\ref{F1}. 
Note that (I), (III) and (IV) satisfy (C-iv) to (C-vi) from the facts that $n_{\min}=o(p^{1/3})$, $\Delta_{iA}=O(||\bmu_{12} ||^2)=O(p^{2/3})$, $\tr(\bSig_i^2)/p \in (0,\infty)$ and $\tr(\bSig_i^4)=o(p^2)$ as $p \to \infty$ for $i=1,2$. 
Thus, Theorem \ref{thm3.1} holds for (I), (III) and (IV). 
We plotted the asymptotic error rates, $\Phi(-\Delta_{(I)}/\delta_{1(I)})$, $\Phi(-\Delta_{1(III)}/\delta_{1(III)})$ and $\Phi(-\Delta_{1(IV)}/\delta_{1(IV)})$ in Figure~\ref{F3}. 
From (\ref{3.2}), we note that $e(1)-e(2)=o(1)$ when $\bSig_1=\bSig_2$. 
Thus, the average error rate, $\overline{e}=\{\overline{e}(1)+\overline{e}(2)\}/2$, is regarded as an estimate of $e(1)$.
We laid $\overline{e}$ for (I), (III) and (IV) by borrowing from Figure~\ref{F1}.
We observed that $\overline{e}$ behaves very close to the asymptotic error rate as expected theoretically.
We also plotted the Bayes error rate, $\Phi(-\Delta_{MD}^{1/2}/2)$. 
We observed that (IV) does not converge to the Bayes error rate when Theorem \ref{thm3.1} is claimed. 
See Section 3.2 for the details. 
\begin{figure}
\begin{centering}
\includegraphics[scale=0.41]{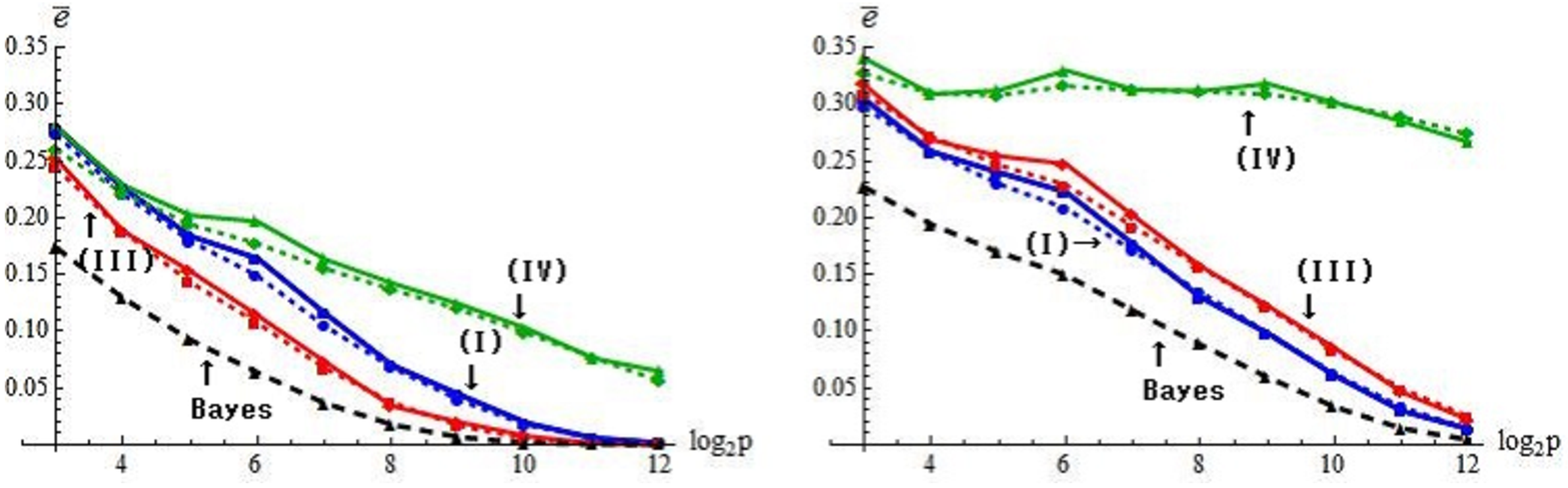}
\\[-1mm]
\ \ 
(a) $\bmu_1=\bze$ and $\bmu_{2}=(1,...,1,0,...,0)^T$ \hspace{0.1cm} (b) $\bmu_1=\bze$ and $\bmu_{2}=(0,...,0,1,...,1)^T$
\caption{The asymptotic error rates (dashed lines) by $\Phi(-\Delta_{(I)}/\delta_{1(I)})$, $\Phi(-\Delta_{1(III)}/\delta_{1(III)})$ and $\Phi(-\Delta_{1(IV)}/\delta_{1(IV)})$, together with the corresponding $\overline{e}$ (solid lines) by (\ref{2.4}), (\ref{2.6}) and (\ref{2.7}) in the setup of Figure~\ref{F1}. 
The Bayes optimal error rate was given by $\Phi(-\Delta_{MD}^{1/2}/2)$.}
\label{F3}
\end{centering}
\end{figure}
As for (I) and (III), the difference of the performances depends on the configuration of $\mu_{ij}$s and $\sigma_{i(j)}$s. 
When $p$ is sufficiently large, we note that $\Delta_{(I)}=\sum_{j=1}^p \mu_{12j}^2<\Delta_{1(III)}=\sum_{j=1}^p \mu_{12j}^2/\sigma_{2(j)}$ for (a) and $\Delta_{(I)}>\Delta_{1(III)}$ for (b) because $\sigma_{2(j)}=0.5+j/(p+1)$, $j=1,...,p$ both for (a) and (b). 
It follows that $\Delta_{(I)}/\delta_{i(I)}<\Delta_{i(III)}/\delta_{i(III)}$ for (a) and $\Delta_{(I)}/\delta_{i(I)}>\Delta_{i(III)}/\delta_{i(III)}$ for (b).
Thus (III) is better than (I) for (a), on the other hand, they trade places for (b). 

When $\bSig_{1}\neq \bSig_{2} $, (II), (III) and (IV) draw information about heteroscedasticity through the difference of $\tr(\bSig_i)$s, $\bSig_{i(d)}$s or $\bSig_i$s, respectively. 
We consider them in the setup of Figure~\ref{F2}. 
For (c), note that $\Delta_{(I)}=0$ but $\Delta_{i(II)}=\Delta_{i(III)}=\Delta_{i(IV)}>c p$ for some constant $c>0$.
(II), (III) and (IV) hold the consistency property even when $n_i$s are fixed because (C-i) to (C-iii) are satisfied. 
Actually, in Figure~\ref{F2}, we observed that the three classifiers gave preferable performances by using the difference of $\tr(\bSig_i)$s, $\bSig_{i(d)}$s or $\bSig_i$s as $p$ increases. 
For (d), note that the difference of $\tr(\bSig_i)$s is smaller than that for (c). 
Actually, in Figure~\ref{F2}, we observed that (II) gives a worse performance for (d) compared to (c). 
On the other hand, (III) gave a better performance compared to (II) because $\Delta_{i(III)}$ is sufficiently larger than $\Delta_{i(II)}$ for (d) when $p$ is large.  
(IV) draws information about heteroscedasticity from the difference of the covariance matrices themselves, so that it gave the best performance in this case. 
However, we note that it is quite difficult to estimate $\bSig_i^{-1}$s feasibly for high-dimensional data. 
See Section 5.2 for the details.
\section{Estimation of the quadratic classifier}
We denote an estimator of $\bA_i$ by $\hat{\bA}_i$.
We consider estimating the quadratic classifier by $W_i(\hat{\bA}_i)$.
\subsection{Preliminary}
Let $||\bM||=\lambda_{\max}^{1/2}(\bM^T \bM)$ for any square matrix $\bM$. 
Let $\kappa$ be a constant such as $\kappa=\Delta_{\min}$ or $\kappa=\delta_{\min}$.
We consider the following condition for $\hat{\bA}_i$s under ($\star$):
\begin{description}
  \item[(C-viii)]\ $p || \hat{\bA}_i-\bA_i||=o_P(\kappa)$ for $i=1,2$.
\end{description}
\begin{pro}
\label{pro4.1}
Assume (C-viii). 
Assume also that $\lambda(\bA_i)\in (0,\infty)$ as $p\to \infty$ for $i=1,2$. 
Then, we have that 
\begin{align}
&W_1(\hat{\bA}_1)-W_2(\hat{\bA}_2)=W_1(\bA_1)-W_2(\bA_2)+o_P(\kappa)\label{4.1} \\
&\mbox{under ($\star$) when $\bx_0\in\pi_i$ for $i=1,2$.} \notag
\end{align}
\end{pro}
When one chooses ${\bA}_i$s as ${\bA}_1 ={\bA}_2\ (={\bA})$, $W(\hat{\bA})$ gives a linear classifier.
We consider the following condition for $\hat{\bA}$ under ($\star$): 
\begin{description}
  \item[(C-ix)]\ $ (p/n_{\min}^{1/2} +p^{1/2}||\bmu_{12}||)|| \hat{\bA}-\bA||=o_P(\kappa)$. 
\end{description}
We have the following result.
\begin{pro}
\label{pro4.2}
Assume (C-ix). 
Then, we have (\ref{4.1}). 
\end{pro}
We note that (C-ix) is milder than (C-viii) from the fact that $||\bmu_{12}||=O(p^{1/2})$. 
Hence, we recommend to use a linear classifier such as (\ref{2.4}) or (\ref{4.5}).
The quadratic classifiers should be used when the difference of covariance matrices is considerably large. 
See Section 4.3 for the details. 
\subsection{Quadratic classifier by $\hat{\bA}_i=\{p/\tr(\bS_i)\} \bI_p$}
We consider the classifier by 
\begin{align}
W_i(\{p/\tr(\bS_{in_i})\} \bI_p)=\frac{p||\bx_0-\overline{\bx}_{in_i}||^2}{\tr(\bS_{in_i})}-\frac{p}{n_i}
+p\log\{\tr(\bS_{in_i})/p \}. \label{4.2}
\end{align}
Note that $\delta_{i}=\delta_{i(II)}$, $\Delta_{i}=\Delta_{i(II)}$ and $\bA_i=\{p/\tr(\bSig_i)\} \bI_p$. 
Here, $\lambda(\bA_i)\in (0,\infty)$ as $p\to \infty$ for $i=1,2$, and (C-viii) naturally holds.
From Corollary \ref{cor2.1} and Proposition \ref{pro4.1}, we have the following result. 
\begin{cor}
\label{cor4.1}
Assume (A-i). 
Assume also (C-i') and (C-ii'). 
Then, for the classification rule by (\ref{1.3}) with (\ref{4.2}), we have (\ref{2.2}). 
\end{cor}
The classifier by (\ref{4.2}) is equivalent to the geometric classifier by \cite{Aoshima:2011}. 
Hereafter, we call the classifier by (\ref{4.2}) the ``geometrical quadratic discriminant analysis (GQDA)". 
Similar to Section 2.2, we have (\ref{2.2}) for GQDA under (A-i) and (\ref{2.3}) even when $n_{\min}$ is fixed. 
If one can assume that $\liminf_{p\to \infty}|\tr(\bSig_1)/\tr(\bSig_2)-1|>0$, we have (\ref{2.2}) for GQDA under (A-i) and (\ref{2.8}) even when $n_{\min}$ is fixed and $\bmu_1=\bmu_2$. 
As for the asymptotic normality, by combining Corollary \ref{cor3.1} with Lemma B.3 given in Appendix B, we have the following result. 
\begin{cor}
\label{cor4.2}
Assume (C-iv') and (C-v').
Assume either (A-i) and (C-vi') or (A-ii) and (C-vii). 
Then, for the classification rule by (\ref{1.3}) with (\ref{4.2}), we have (\ref{3.2}) under 
$(\tr(\bSig_{1})/\tr(\bSig_{2})-1)^2\tr(\bSig_{\max}^{2})=o(n_{\min}\delta_{\min(II)}^2)$ as $m\to \infty$, 
where $\delta_{\min(II)}=\min\{\delta_{1(II)},\delta_{2(II)} \}$.
\end{cor}
Now, we compare DBDA with GQDA. 
We have that
\begin{align*}
\hat{\Delta}_{(I)}=&||\overline{\bx}_{1n_1}-\overline{\bx}_{2n_2}||^2-\tr(\bS_{1n_1})/n_1-\tr(\bS_{2n_2})/n_2 \quad \mbox{and} \\
\hat{\Delta}_{i(II)}=&\frac{ p}{\tr(\bS_{i'n_{i'}})}\Big[\hat{\Delta}_{(I)}+\tr(\bS_{in_i})-\tr(\bS_{i'n_{i'}})+\tr(\bS_{i'n_{i'}})\log \Big\{\frac{ \tr(\bS_{i'n_{i'}})}{\tr(\bS_{in_i})}\Big\} \Big]
\end{align*}
for $i=1,2\ (i' \neq i)$. 
We note that $E(\hat{\Delta}_{(I)})=\Delta_{(I)}$. 
From (\ref{3.2}) and Remark \ref{rem3.1}, if $\hat{\Delta}_{i(II)}\tr(\bS_{i'n_{i'}})/p$ is sufficiently larger than $\hat{\Delta}_{(I)}$ for some $i$, we recommend to use GQDA. 
Otherwise one may use DBDA free from (A-i). 
See Corollary 2.1 for the details. 
\subsection{Quadratic classifier by $\hat{\bA}_i=\bS_{in_i(d)}^{-1}$}
Let $\bS_{in_i(d)}=\mbox{diag}(s_{in_i(1)},...,s_{in_i(p)})$ for $i=1,2$. 
We consider the classifier by 
\begin{align}
W_i(\bS_{in_i(d)}^{-1})=\sum_{j=1}^p\Big(\frac{(x_{0j}-\overline{x}_{ijn_i})^2}{ s_{in_i(j)} }-\frac{1}{n_i}+\log{s_{in_i(j)} }\Big). \label{4.3}
\end{align}
Note that $\delta_{i}=\delta_{i(III)}$, $\Delta_{i}=\Delta_{i(III)}$ and $\bA_i=\bSig_{i(d)}^{-1}$.
\cite{Dudoit:2002} considered the quadratic classifier without the bias correction term. 
That was called the diagonal quadratic discriminant analysis (DQDA). 
Hereafter, we call the classifier by (\ref{4.3}) ``DQDA-bc". 
Let $\eta_{i(j)}=\Var\{(x_{ijk}-\mu_{ij})^2\}$ for $i=1,2$, and $j=1,...,p\ (k=1,...,n_i)$.
Since $\hat{\bA}_i=\bS_{in_i(d)}^{-1}$ does not satisfy (C-viii) in that shape, we consider the following assumption:
\begin{description}
\item[(A-iii)] $\eta_{i(j)}\in (0,\infty)$ as $p\to \infty$ and 
$\displaystyle \limsup_{p\to \infty} E \big\{ \exp \big(t_{ij}|x_{ijk}-\mu_{ij}|^2/\eta_{i(j)}^{1/2} \big) \big\}<\infty$ for some $t_{ij}>0$, $i=1,2$, and $j=1,...,p\ (k=1,...,n_i)$.
\end{description}
Note that (A-iii) holds when $\pi_i$ has $N_p(\bmu_i,\bSig_i)$ for $i=1,2$. 
From Corollary \ref{cor2.1} and Proposition \ref{pro4.1}, we have the following result. 
\begin{cor}
\label{cor4.3}
Assume (A-i) and (A-iii). 
Assume also (C-ii'). 
Then, for the classification rule by (\ref{1.3}) with (\ref{4.3}), we have (\ref{2.2}) under the condition that 
\begin{equation}
\frac{p^2 \log{p} }{n_{\min} \Delta_{\min(III)}^2}=o(1).\label{4.4}
\end{equation}
\end{cor}
Note that (C-i') holds under (\ref{4.4}). 
From the fact that $\Delta_{i(III)}=O(p)$, it follows that $n_{\min}^{-1}\log {p}=o(1)$ under (\ref{4.4}).
Similar to Section 2.2, if one can assume that 
$\liminf_{p\to \infty} $
$||\bmu_{12}||^2/p>0$ or $\liminf_{p\to \infty} \sum_{j=1}^p|\sigma_{1(j)}/\sigma_{2(j)}-1|/p>0$,  
DQDA-bc holds (\ref{2.2}) under (A-i), (A-iii), (\ref{2.8}) and $n_{\min}^{-1}\log {p}=o(1)$. 
When $\Delta_{\min(III)}$ is not sufficiently large, say $\Delta_{\min(III)}=O(p^{1/2})$, we can claim Corollary \ref{cor4.3} in high-dimension, large-sample-size settings such as $n_{\min}/p\to \infty$. 
In Section 5, we shall provide a DQDA type classifier by feature selection and show that it has the consistency property even when $n_{\min}/p\to 0$ and $\Delta_{\min(III)}$ is not sufficiently large. 

%
Next, we consider the pooled sample diagonal matrix, 
$$
\bS_{n(d)}=\frac{\sum_{i=1}^2(n_i-1) \bS_{in_i(d)}}{\sum_{i=1}^2n_i-2}.
$$ 
Note that $E(\bS_{n(d)})=\sum_{i=1}^2(n_i-1)\bSig_{i(d)}/(\sum_{i=1}^2n_i-2)\ (\mbox{hereafter called }\bSig_{(d)})$. 
When $\bSig_{1(d)}=\bSig_{2(d)}$, it follows that $\bSig_{(d)}=\bSig_{i(d)},$ $i=1,2$. 
Let us write $\bS_{n(d)}=\mbox{diag}(s_{n(1)},...,$
$s_{n(p)})$ and $\bSig_{(d)}=\mbox{diag}(\sigma_{(1)},...,\sigma_{(p)})$. 
We consider the classifier by 
\begin{align}
W_i(\bS_{n(d)}^{-1})=\sum_{j=1}^p\Big(\frac{(x_{0j}-\overline{x}_{ijn_i})^2}{ s_{n(j)} }-\frac{s_{in_i(j)}}{n_i s_{n(j)}}
\Big). \label{4.5}
\end{align}
We note that the classification rule by (\ref{1.3}) with (\ref{4.5}) becomes a linear classifier. 
\cite{Bickel:2004} and \cite{Dudoit:2002} considered the linear classifier without the bias correction term. 
That was called the diagonal linear discriminant analysis (DLDA). 
Hereafter, we call the classifier by (\ref{4.5}) ``DLDA-bc". 
Although \cite{Huang:2010} gave bias corrected versions of DLDA and DQDA, they considered a bias correction only when $\pi_i$s are Gaussian. 
We note that $\Delta_{1}=\Delta_{2}=\sum_{j=1}^p\mu_{12j}^2/\sigma_{(j)}\ (\mbox{hereafter called } \Delta_{(III')})$ and $\bA_1=\bA_2=\bSig_{(d)}^{-1}$. 
Then, by combining Theorem \ref{thm2.1} with Propositions \ref{pro2.1} and \ref{pro4.2}, we have the following result. 
\begin{cor}
\label{cor4.4}
Assume (A-iii). 
Assume also (C-i') and (C-ii'). 
Then, for the classification rule by (\ref{1.3}) with (\ref{4.5}), we have (\ref{2.2}) under the condition that 
\begin{equation}
\frac{p \log{p}  }{n_{\min} \Delta_{(III')} }=o(1).\label{4.6}
\end{equation}
\end{cor}
Under $n_{\min}^{-1}\log{p}=o(1)$, one may claim that (\ref{4.6}) is milder than (\ref{4.4}) if $\Delta_{\min(III)}$ and $\Delta_{(III')}$ are of the same order. 
Hence, we recommend to use DQDA-bc when $\Delta_{\min(III)}$ is considerably larger than $\Delta_{(III')}$. 
Otherwise one may use DLDA-bc even when $\bSig_{i(d)}$s are not common.
We shall improve DQDA-bc by feature selection in Section 5. 
%
\subsection{Quadratic classifier by $\hat{\bA}_i=\bS_{in_i}^{-1}$}
In this section, we consider high-dimension, large-sample-size situations such as $n_{\min}/p\to \infty$ as $p\to \infty$ and discuss the classifier by 
\begin{align}
W_i(\bS_{in_i}^{-1})=(\bx_0-\overline{\bx}_{in_i})^T\bS_{in_i}^{-1}(\bx_0-\overline{\bx}_{in_i})
-p/n_i+\log|\bS_{in_i}|. 
\label{4.7}
\end{align}
Note that $\delta_{i}=\delta_{i(IV)}$, $\Delta_{i}=\Delta_{i(IV)}$ and $\bA_i=\bSig_i^{-1}$.
Let $\eta_{i(rs)}=\Var\{(x_{irk}-\mu_{ir})(x_{isk}-\mu_{is}) \}$ 
for $i=1,2$, and $r,s=1,...,p\ (k=1,...,n_i)$. 
From Theorem \ref{thm2.1} and Proposition \ref{pro4.1}, we have the following result. 
\begin{cor}
\label{cor4.5}
Assume (A-i) and (A-iii). 
Assume also $\lambda(\bSig_i)\in (0,\infty)$ as $p\to \infty$ and $\liminf_{p\to \infty} \eta_{i(rs)}>0$ for all $r,s$; $i=1,2$. 
Then, for the classification rule by (\ref{1.3}) with (\ref{4.7}), we have (\ref{2.2}) under the conditions that $p^{1/2}/ \Delta_{\min(IV)}=o(1)$ and 
\begin{equation}
\frac{p^4 \log{p} }{n_{\min} \Delta_{\min(IV)}^2}=o(1).\label{4.8}
\end{equation}
\end{cor}

From the fact that $\Delta_{i(IV)}=O(p)$ when $\lambda(\bSig_i)\in (0,\infty)$ as $p\to \infty$ for $i=1,2$, it follows that $n_{\min}^{-1}p^2\log{p}=o(1)$ under (\ref{4.8}). 
Thus, the classification rule by (\ref{1.3}) with (\ref{4.7}) can claim the consistency property when $n_{\min}^{-1} p^2 \log {p}=o(1)$. 
However, the condition ``$n_{\min}^{-1} p^2 \log {p}=o(1)$" is quite strict for high-dimensional data.  
In Section 5, we shall discuss a classifier by sparse inverse covariance matrix estimation when $n_{\min}/p \to 0$. 
\section{Quadratic classifiers by feature selection and sparse inverse covariance matrix estimation}
In this section, we propose a new quadratic classifier by feature selection for (\ref{4.3}) and discuss a quadratic classifier by sparse inverse covariance matrix estimation for (\ref{4.7}). 
\subsection{Quadratic classifier after feature selection}
We consider applying a variable selection procedure to classification.
\cite{Fan:2008} proposed the feature annealed independent rule based on the difference of mean vectors. 
However, we give a different type of feature selection by using both the differences of mean vectors and covariance matrices. 
We have that 
$$
\Delta_{1(III)}+\Delta_{2(III)}=\sum_{j=1}^p \Big( \frac{\mu_{12j}^2+\sigma_{1(j)} }{\sigma_{2(j)}}+\frac{\mu_{12j}^2+\sigma_{2(j)} }{\sigma_{1(j)}}-2\Big).
$$
Let $ \theta_j=(\mu_{12j}^2+\sigma_{1(j)})/(2\sigma_{2(j)})+(\mu_{12j}^2+\sigma_{2(j)})/(2\sigma_{1(j)})-1$ for $j=1,...,p$. 
Note that $\Delta_{1(III)}+\Delta_{2(III)}=2\sum_{j=1}^p \theta_j$. 
Also, note that $\theta_j>0$ when $\mu_{1j}\neq \mu_{2j}$ or $\sigma_{1(j)}\neq \sigma_{2(j)}$. 
Now, we give an estimator of $\theta_j\ (j=1,...,p)$ by 
$$
\hat{\theta}_j=\frac{(\overline{x}_{1jn_1}-\overline{x}_{2jn_2})^2+s_{1n_1(j)}}{2s_{2n_2(j)}}+
\frac{(\overline{x}_{1jn_1}-\overline{x}_{2jn_2})^2+s_{2n_2(j)}}{2s_{1n_1(j)}}-1.
$$
Then, we have the following result. 
\begin{thm}
\label{thm5.1}
Assume (A-iii). Assume also $n_{\min}^{-1}\log{p}=o(1)$. 
Then, we have that as $p\to \infty$ 
$$
\max_{j=1,...,p}| \hat{\theta}_j-{\theta}_j|=O_P\{ (n_{\min}^{-1} \log{p})^{1/2} \}.
$$
\end{thm}
Let $\bD=\{j\ |\ \theta_j>0\ \mbox{for $j=1,...,p$} \}$ and $p_*=\#\bD$, where $\#\bS$ denotes the number of elements in a set $\bS$. 
Let $\xi=(n_{\min}^{-1}\log{p})^{1/2}$. 
We select a set of significant variables by 
\begin{equation}
\widehat{\bD}=\{j\ |\  \hat{\theta}_j>\xi^{\gamma}\ \mbox{for $j=1,...,p$} \}, 
\label{5.1}
\end{equation}
where $\gamma \in (0,1)$ is a chosen constant. 
Then, from Theorem \ref{thm5.1}, we have the following result. 
\begin{cor}
\label{cor5.1}
Assume (A-iii) and $n_{\min}^{-1}\log{p}=o(1)$. 
Assume also $\liminf_{p\to \infty}\theta_j>0$ for all $j\in \bD$. 
Then, we have that $P(\bD=\widehat{\bD})\to 1$ as $p\to \infty$.
\end{cor}
\begin{rem}
\label{rem5.1}
As for $l\ (\ge 3)$-class classification, one may consider $\hat{\theta}_j$ such as $\hat{\theta}_j=\sum_{i\neq i'}^k\{(\overline{x}_{ijn_i}-\overline{x}_{i'jn_{i'}})^2+s_{in_i(j)}\}/\{k(k-1)s_{i'n_{i'}(j)}\}-1$ for $j=1,...,p$.
\end{rem}

Now, we consider a classifier using only the variables in $\widehat{\bD}$. 
We define the classifier by 
\begin{align}
W_i( \bS_{in_i(d)}^{-1})_{FS}=\sum_{j \in { \mbox{\scriptsize $\widehat{\bD}$} } } \Big(\frac{(x_{0j}-\overline{x}_{ijn_i})^2}{ s_{in_i(j)} }-\frac{1}{n_i}+\log{s_{in_i(j)} }\Big) \label{5.2}
\end{align}
for $i=1,2$. 
We consider the classification rule by (\ref{1.3}) with (\ref{5.2}).
We call this feature selected DQDA ``FS-DQDA". 
Let us write that $\bx_{i*k}=(x_{ij_1k},....,x_{ij_{p_*}k})^T$ for all $i,k$, where $\bD=\{j_1,...,j_{p_*} \}$. 
Let $\bSig_{i*}=\Var(\bx_{i*k})$ for $i=1,2\ (k=1,...,n_i)$. 
Then, from Theorem \ref{thm2.1} and Corollary \ref{cor5.1}, we have the following result. 
\begin{cor}
\label{cor5.2}
Assume (A-i) and (A-iii). 
Assume also $\lambda_{\max}(\bSig_{i*})=o(p_*)$ for $i=1,2$, and $\liminf_{p\to \infty}\theta_j>0$ for all $j\in \bD$.
Then, for the classification rule by (\ref{1.3}) with (\ref{5.2}), we have (\ref{2.2}) under $n_{\min}^{-1}\log{p}=o(1)$.
\end{cor}
By comparing Corollary \ref{cor5.2} with \ref{cor4.3}, note that the condition ``$n_{\min}^{-1}\log{p}=o(1)$" is much milder than (\ref{4.4}). 
Thus we recommend FS-DQDA more than DQDA-bc (or the original DQDA). 
For a choice of $\gamma\in (0,1)$ in (\ref{5.1}), we recommend applying cross-validation procedures or choosing a constant such as $\gamma=0.5$ because Corollary \ref{cor5.2} is claimed for any $\gamma\in (0,1)$. 
In addition, we emphasize that the computational cost of FS-DQDA is quite low even when $p\ge 10,000$. 
\subsection{Quadratic classifier by sparse inverse covariance matrix estimation}
We consider applying a sparse estimation of inverse covariance matrices to classification. 
\cite{Bickel:2008} gave a sparse estimator of $\bSig_i^{-1}$. 
Let $\sigma_{i(st)}$ be the $(s,t)$ element of $\bSig_{i}$ for $s,t=1,...,p\ (i=1,2)$.
A sparsity measure of $\bSig_i\ (i=1,2)$ is given by
$
c_{p,h_i}=\max_{1\le t \le p} \sum_{s=1}^p|\sigma_{i(st)}|^{h_i}
$
for $0\le h_i <1$, where $0^0$ is defined to be $0$. 
Note that $\lambda_{\max}(\bSig_i)\le M c_{p,h_i}$ for some constant $M>0$.
If $c_{p,h_i}$ is much smaller than $p$ for a constant $h_i\in [0,1)$, $\bSig_i$ is considered as sparse in the sense that many elements of $\bSig_i$ are very small. 
See Section 3 in \cite{Shao:2011} for the details. 
Let $I(\cdot)$ be the indicator function. 
A thresholding operator is defined by 
$
T_{\tau}(\bM)=[m_{st}I(|m_{st}|\ge \tau)] 
$
for any $\tau>0$ and any symmetric matrix $\bM=[m_{st}]$. 
Let $\tau_{n_i}=M'(n_i^{-1}\log{p})^{1/2}$ for some constant $M'>0$. 
Then, \cite{Bickel:2008} gave the following result. 
\begin{thm}
\label{thm5.2}
Assume (A-iii), $n_i^{-1}\log{p}=o(1)$ and $\liminf_{p\to \infty}\lambda_{\min}(\bSig_i)>0$. 
For a sufficiently large $M'(>0)$, it holds that as $p\to \infty$ 
$$
|| \{T_{\tau_{n_i}}(\bS_{in_i})\}^{-1}-\bSig_i^{-1}||=O_P\Big(c_{p,h_i} (n_i^{-1}\log{p})^{(1-h_i)/2}\Big).
$$
\end{thm}
\begin{rem}
Theorem \ref{thm5.2} is obtained by Theorem 1 and Section 2.3 in \cite{Bickel:2008}. 
\end{rem}
We use $\hat{\bA}_i=\{T_{\tau_{n_i}}(\bS_{in_i})\}^{-1}$ as an estimator of $\bSig_i^{-1}$ and consider the classifier by $W_i( \{T_{\tau_{n_i}}(\bS_{in_i})\}^{-1} )$.
By combining Theorem \ref{5.2} and Proposition \ref{pro4.1}, if it holds that $\lambda(\bSig_i)\in (0,\infty)$ as $p\to \infty$ and 
\begin{equation}
\frac{pc_{p,h_i} (n_i^{-1}\log{p})^{(1-h_i)/2}}{\Delta_{\min(IV)}}=o_P(1),
\label{5.3}
\end{equation}
the classification rule by (\ref{1.3}) with $W_i( \{T_{\tau_{n_i}}(\bS_{in_i})\}^{-1} )$ has (\ref{2.2}) under some regularity conditions. 
When $\bSig_i$s are sparse as $c_{p,h_i}=O(1)$ for some $h_i (i=1,2)$ and $\liminf_{p\to \infty}\Delta_{\min(IV)}/p>0$, (\ref{5.3}) holds in HDLSS situations such as $n_{\min}^{-1}\log {p}=o(1)$. 
\cite{Shao:2011} and \cite{Li:2015} considered a linear and a quadratic classifier by the sparse estimation of $\bSig_i^{-1}$s under some sparsity conditions. 
On the other hand, \cite{Cai:2011a} gave the constrained $\ell_1$-minimization for inverse matrix estimation (CLIME). 
One may apply the CLIME to the classification rule by (\ref{1.3}). 
However, one should note that the computational cost for the sparse estimation of $\bSig_i^{-1}$s is extremely high even when $p \approx 1,000$. 
It is quite unrealistic to apply the estimation to classification when $p$ is very high as $p \ge 10,000$.
Also, the sparsity condition ``$\lambda(\bSig_i)\in(0,\infty)$ as $p\to\infty$" is quite severe for high-dimensional data.
In actual data analyses, we often encounter the situation that $\lambda_{ij}\to\infty$ as $p\to\infty$ for the first several $j$s.
See \cite{Yata:2013b} for the details.
\subsection{Simulation}
We used computer simulations to compare the performance of the classifiers: DBDA by (\ref{2.4}), GQDA by (\ref{4.2}), DLDA-bc by (\ref{4.5}), DQDA-bc by (\ref{4.3}) and FS-DQDA by (\ref{5.2}). 
We did not compare the classifiers with the one given by sparse estimation of $\bSig_i^{-1}$s such as $W_i( \{T_{\tau_{n_i}}(\bS_{in_i})\}^{-1} )$ in Section 5.2 because the computational cost of the sparse estimation is very high when $p$ is large. 
Thus we considered the classifier by (\ref{2.7}) instead of using the sparse estimation, provided that $\bSig_i$s were known. 
We set $\gamma=0.5$ in (\ref{5.1}). 
We considered $p_*=\lceil p^{1/2} \rceil$. 
We generated $\bx_{ik}-\bmu_i$, $k=1,2,...,\ (i=1,2)$ independently from (i) $N_p(\bze, \bSigma_i)$ or (ii) a $p$-variate $t$-distribution, $t_p(\bze,\bSig_i,\nu)$ with mean zero, covariance matrix $\bSig_i$ and degrees of freedom $\nu$. 
We set $p=2^{s},\ s=3,...,10$ for (i), and $p=500$ and $\nu=4s,\ s=1,...,8$ for (ii). 
We set $\bmu_1=\bze$, $\bmu_{2}=(0,...,0,1,...,1)^T$ whose last $p_*$ elements are $1$ and $\bSig_1=\bB_1( 0.3^{|i-j|^{1/3}})\bB_1$, where $\bB_1$ is defined in Section 1. 
Let $\bB_2=\mbox{diag}(1,...,1,2^{1/2},...,2^{1/2})$ whose last $p_* $ diagonal elements are $2^{1/2}$.
We considered four cases:
\begin{flushleft}
(a) $n_1=10$, $n_2=20$ and $\bSig_2=\bSig_1$ for (i) $N_p(\bze, \bSigma_i)$; \\
(b) $n_1=\lceil (\log{p})^2 \rceil$, $n_2=2n_1$ and $\bSig_2=\bSig_1$ for (i) $N_p(\bze, \bSigma_i)$; \\
(c) $n_1=\lceil (\log{p})^2 \rceil$, $n_2=2n_1$ and $\bSig_2=\bB_2 \bSig_1 \bB_2$ for (i) $N_p(\bze, \bSigma_i)$; \\
and (d) $n_1=\lceil (\log{p})^2 \rceil$, $n_2=2n_1$ and $\bSig_2=\bB_2 \bSig_1 \bB_2$ for (ii) $t_p(\bze,\bSig_i,\nu)$. 
\end{flushleft}
It holds that $n_{\min}^{-1}\log{p}=o(1)$ for (b), (c) and (d), $\liminf_{p\to \infty}\Delta_{\min}/p_*>0$ for (a) to (d), and $\liminf_{p\to \infty}|\tr(\bSig_1)-\tr(\bSig_2)|/p_*>0$ for (c) and (d). 
Similar to Section 1, we calculated the average error rate, $\overline{e}$, by $2000$ replications and plotted the results in Figure~\ref{F4} (a) to (d). 

We observed from (a) in Figure~\ref{F4} that DBDA and GQDA give preferable performances when $n_i$s are fixed. 
DLDA-bc gave a moderate performance because $\bSig_1=\bSig_2$. 
However, the other classifiers did not give preferable performances when $p$ is large. 
This is probably due to the consistency property of those classifiers (except (\ref{2.7})) which is claimed under at least $n_{\min}^{-1}\log{p}=o(1)$. 
Actually, as for (b), the other classifiers gave moderate performances because $n_{\min}^{-1}\log{p}=o(1)$. 
Thus we do not recommend to use quadratic classifiers including all the elements (or the diagonal elements) of sample covariance matrices, such as DQDA-bc and FS-DQDA, when the condition ``$n_{\min}^{-1}\log{p}=o(1)$" is not satisfied. 
When $n_{\min}^{-1}\log{p}\ne o(1)$ or $n_i$s are fixed, we recommend to use DBDA and GQDA. 
On the other hand, FS-DQDA gave a good performance for (c) as $p$ increases because the difference of the covariance matrices becomes large as $p$ increases. 
We note that from Corollary \ref{cor5.2} FS-DQDA holds the consistency property for (c). 
However, DQDA-bc did not give a preferable performance because $\Delta_{\min}(III)=O(p^{1/2})$, so that DQDA-bc does not hold the consistency property from Corollary \ref{cor4.3}. 
We note that $\bSig_1 \neq \bSig_2$ but $\Delta_{(I)}/\delta_{i(I)} \approx \Delta_{i(II)}/\delta_{i(II)}$ for (c). 
Thus GQDA gave a similar performance to DBDA for (c). 
As for (d), DBDA gave a preferable performance even when $\nu$ is small because DBDA holds the consistency property without (A-i). 
The other classifiers did not give preferable performances when $\nu$ is small. 
However, these classifiers gave moderate performances when $\nu$ becomes large because $t_p(\bze,\bSig_i,\nu)\Rightarrow N_p(\bze, \bSigma_i)$ as $\nu\to \infty$. 
Especially, FS-DQDA gave a good performance when $\nu$ is not small. 
This is probably because FS-DQDA has smaller variance by feature selection, such as $p_*/p\to 0$, compared to the other classifiers. 

Throughout the simulations, the classifier by (\ref{2.7}) did not give preferable performances in spite that $\bSig_i$s are known. 
See Section 3.2 for theoretical reasons. 
Therefore, it is likely that the classifier by $W_i( \{T_{\tau_{n_i}}(\bS_{in_i})\}^{-1} )$ gives poor performances for the high-dimensional settings. 
\begin{figure}
\begin{centering}
\includegraphics[scale=0.41]{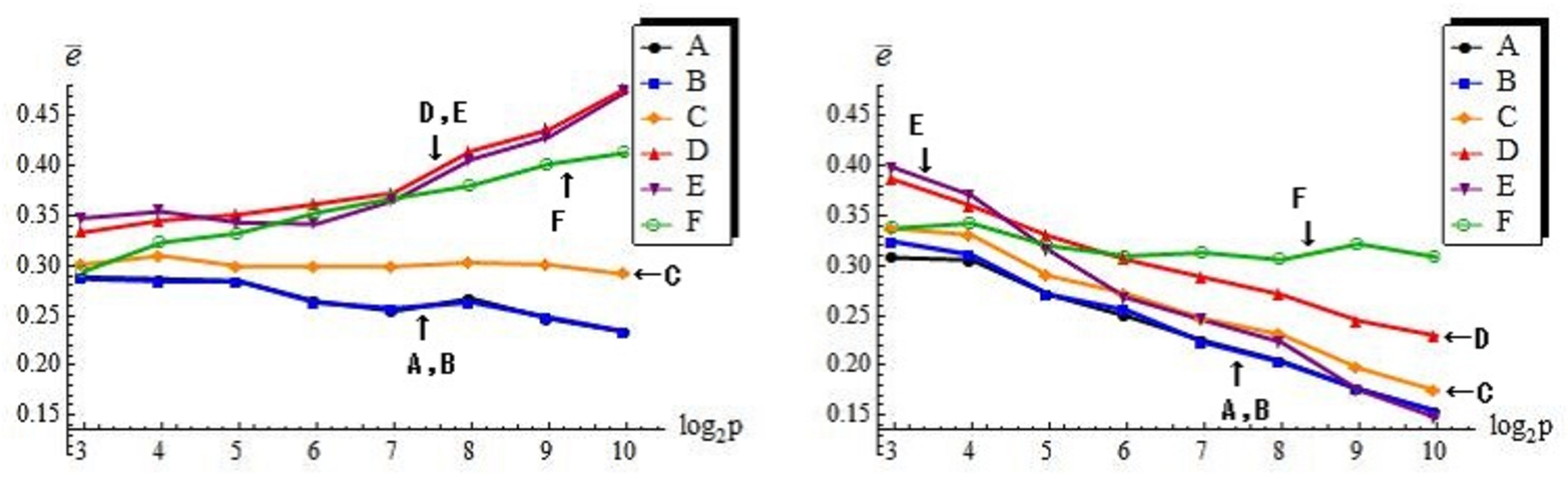}  \\[-1mm]
\quad (a) $\bSig_1=\bSig_2$ ($n_i$s are fixed) \hspace{1.2cm} (b) $\bSig_1=\bSig_2$ \hspace{2cm} \ \\[2mm]
\includegraphics[scale=0.41]{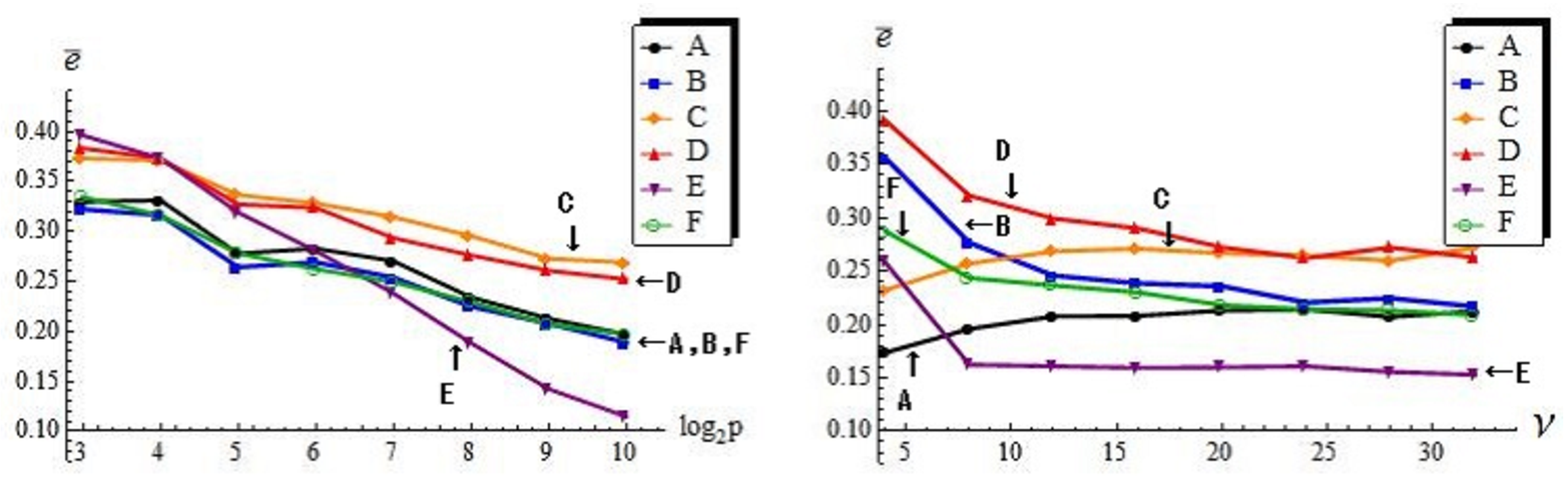} \\[-1mm]
\quad \quad \ (c) $\bSig_1\neq \bSig_2$ \hspace{3.1cm} (d) $\bSig_1\neq \bSig_2$ (non-Gaussian case)
\caption{The average error rates of the classifiers: A: DBDA, B: GQDA, C: DLDA-bc, D: DQDA-bc, E: FS-DQDA, and F: the classifier by (\ref{2.7}).} 
\label{F4}
\end{centering}
\end{figure}
\section{Example: Leukemia data sets}
We first analyzed gene expression data given by \cite{Golub:1999} in which the data set consists of $7129\ (=p)$ genes and $72$ samples. 
We had 2 classes of leukemia subtypes, that is, $\pi_1$: acute lymphoblastic leukemia (ALL) ($47$ samples) and $\pi_2$: acute myeloid leukemia (AML) ($25$ samples). 
The data set consisted of two sets as $38$ training samples (ALL: $27$ samples and AML: $11$ samples) and $34$ test samples (ALL: $20$ samples and AML: $14$ samples). 
Note that $\bS_{1n_1(d)}=\bS_{2n_2(d)}$ if each sample has unit variance.
Thus we did not standardize each sample so as to have unit variance. 

First, we checked several sparsity conditions. 
We standardized each sample by $\bx_{ik}$
$/\{\sum_{l=1}^2\tr(\bS_{ln_l})/(2p)\}^{1/2}$ for all $i,k$, so that $\tr(\bS_{1n_1})/2+\tr(\bS_{2n_2})/2=p$. 
By using all the samples (i.e., $72$ samples), we calculated that
\begin{equation}
\hat{\Delta}_{(I)}=2060\ (=0.289p), \label{6.1}
\end{equation}
where $\hat{\Delta}_{(I)}$ is given in Section 4.2. 
Note that $E(\hat{\Delta}_{(I)})=||\bmu_{12}||^2$. 
From this observation, we concluded that $\bmu_{12}$ is non-sparse.
Next, we considered an estimator of $||\bSig_{12}||_F^2=\sum_{i=1}^2\tr(\bSig_i^2)-2\tr(\bSig_1\bSig_2)$ by $\hat{\Delta}_{\Sigma}=\sum_{i=1}^2W_{in_i}-2\tr(\bS_{1n_1}\bS_{2n_2})$ having $W_{in_i}$s defined by (16) in \cite{Aoshima:2014}.
Here, $W_{in_i}$ is an unbiased estimator of $\tr(\bSig_i^2)$, so that $E(\hat{\Delta}_{\Sigma})=||\bSig_{12}||_F^2$. 
We calculated that
$$
\hat{\Delta}_{\Sigma}=9.77\times 10^5\ (=137p).
$$
From this observation, we concluded that $\bSig_{12}$ is non-sparse. 
Therefore, the Bayes error rates of this data set are probably close to $0$.
Also, we calculated 
\begin{equation}
(\acute{\lambda}_{\max}(\bSig_1),\acute{\lambda}_{\max}(\bSig_2))=(1223,1457)\ (=(0.172p,0.204p)),
\label{6.2}
\end{equation}
where $\acute{\lambda}_{\max}(\bSig_i)$ is an estimate of the largest eigenvalue due to the noise-reduction methodology by \cite{Yata:2013b}. 
We concluded that ``$\lambda(\bSig_i)\in (0,\infty) $ as $p\to \infty$" does not hold and $\bSig_i$s are non-sparse because $\lambda_{\max}(\bSig_i)$s are very large. 
Therefore, we do not recommend to apply the classifier by sparse estimation of $\bSig_i^{-1}$, such as $W_i(\{T_{\tau_{n_i}}(\bS_{in_i})\}^{-1})$. 
Actually, we did not use any classifiers by sparse estimation of $\bSig_i^{-1}$ in this section. 
Also, note that the computational cost for the sparse estimation of $\bSig_i^{-1}$ is very high when $p$ is large.

We constructed the classifiers: DBDA, GQDA, DLDA-bc, DQDA-bc and FS-DQDA, by using the training samples of sizes $n_1=27$ and $n_2=11$, and checked the accuracy by using the test samples from each $\pi_i$. 
Throughout this section, we set $\gamma=0.5$ in (\ref{5.1}) for FS-DQDA. 
We compared the classifiers with the hard-margin linear support vector machine (HM-LSVM).
See \cite{Vapnic:1999} for the details.
Note that the data sets are linearly separable by a hyperplane because $p>n_1+n_2$.
We emphasize that the computational cost of DBDA, GQDA, DLDA-bc, DQDA-bc or FS-DQDA is as low as HM-LSVM even when $p\ge 10,000$.   
We summarized misclassification rates in the first block of Table~\ref{Tab1}. 
We note that $n_{\min}=11$ and $n_{\min}^{-1}\log{p}=0.81$, so that ``$n_{\min}^{-1}\log{p}=o(1)$" does not hold. 
That is probably the reason why DLDA-bc, DQDA-bc and FS-DQDA seem to lose the consistency property. 
See Sections 4 and 5 for the details. 
On the other hand, DBDA and GQDA gave reasonable performances even when $n_i$s are small and seem to hold the consistency property.
We calculated $\tr(\bS_{1n_1})/\tr(\bS_{2n_2})=0.989$ and $(\hat{\Delta}_{i(II)}\tr(\bS_{i'n_{i'}})/p)/\hat{\Delta}_{(I)} \approx 1$ for $i \neq i'$. 
The difference of the trace of the covariance matrices is small and this is probably the reason why DBDA gave a preferable performance. 
See Section 4.2 for the details. 
In addition, HM-LSVM also gave a preferable performance. 
See \cite{Hall:2005} for the consistency property of HM-LSVM. 
For this data set, \cite{Cai:2011b} summarized misclassification rates for several other classifiers including a sparse linear classifier called LPD. 
See Table 6 in \cite{Cai:2011b} for the performances of the other classifiers.
Note that LPD has the Bayes error rates asymptotically under several sparsity conditions. 
We observed that DBDA and GQDA gave the same accuracy as LPD. 
This is probably because the sparsity conditions do not hold for this data set, so that the Bayes error rates are almost $0$.
However, the computational cost for DBDA and GQDA is much lower than LPD. 
\begin{table}
\begin{center}
\caption{Error rates of the classifiers for samples from \cite{Golub:1999}.}
\label{Tab1}
{\footnotesize
\begin{tabular}{c|cccccc}
\hline
\\[-3mm]
 Classifier & DBDA & GQDA & DLDA-bc & DQDA-bc & FS-DQDA & HM-LSVM \\ 
\hline \\[-3mm]
&\multicolumn{6}{c}{Test samples (ALL: 20 and AML: 14)}   \\[0mm]
Error rate & 1/34 & 1/34 & 5/34 & 2/34 & 3/34  & 1/34 \\[1mm]
\hline \\[-3mm]
&\multicolumn{6}{c}{LOOCV of samples (ALL: 47 and AML: 25)}   \\[0mm]
Error rate & 3/72 & 6/72 & 11/72 & 1/72 & 0/72 & 2/72  \\[1mm]
\hline    
\end{tabular}
}
\end{center}
\end{table}

Next, by using all the samples (i.e., $72$ samples), we checked the accuracy of the classifiers by the leave-one-out cross-validation (LOOCV). 
We summarized misclassification rates in the second block of Table~\ref{Tab1}. 
We note that $n_{\min}=24$ and $n_{\min}^{-1}\log{p}=0.37$ or $n_{\min}=25$ and $n_{\min}^{-1}\log{p}=0.35$ in this case, so that $n_{\min}^{-1}\log{p}$ is a little small. 
We observed that DQDA-bc and FS-DQDA give preferable performances. 
On the other hand, DLDA-bc gave a poor performance because it does not draw information about heteroscedasticity. 
For other classifiers, \cite{Tan:2005} summarized results of the LOOCV for this data set. 

Finally, we analyzed gene expression data given by \cite{Armstrong:2002} in which the data set consists of $12582\ (=p)$ genes and $72$ samples. 
We had 3 classes of leukemia subtypes: acute lymphoblastic leukemia (ALL: $24$ samples), mixed-lineage leukemia (MLL: $20$ samples), and acute myeloid leukemia (AML: $28$ samples).
We considered three cases: (a) ALL and MLL, (b) ALL and AML, and (c) MLL and AML. 
We standardized each sample by $\bx_{ik}/\{\sum_{l=1}^3\tr(\bS_{ln_l})/(3p)\}^{1/2}$ for all $i,k$, as before. 
Then, we calculated ($\hat{\Delta}_{(I)}$, $\hat{\Delta}_{\Sigma}$) for the three cases. 
We summarized ($\hat{\Delta}_{(I)}$, $\hat{\Delta}_{\Sigma}$)s in Table~\ref{Tab2}. 
\begin{table}
\begin{center}
\caption{Estimates of ($||\bmu_{12}||^2$,\ $||\bSig_{12}||_F^2$) by ($\hat{\Delta}_{(I)}$, $\hat{\Delta}_{\Sigma}$) for \cite{Armstrong:2002}.}
\label{Tab2}
{\footnotesize
\begin{tabular}{c|ccc}
\hline
\\[-3mm]
 Case  & \ (a) ALL and MLL \ &  \ (b) ALL and AML \ &  \ (c) MLL and AML \ \\ 
\hline \\[-3mm]
$||\bmu_{12}||^2$  & 4076\ ($=0.324p$) & 15050\ ($=1.2p$) & 8546\ ($=0.679p$) \\[1mm]
\hline \\[-2mm]
$||\bSig_{12}||_F^2$ & $1.12\times 10^8$\ ($=8863p$)  & $5.49\times 10^6$\ ($=436p$) & $1.16\times 10^8$\ ($=9192p$) \\[1mm]
\hline    
\end{tabular}
}
\end{center}
\end{table}
From Table~\ref{Tab2}, we concluded that $\bmu_{12}$ and $\bSig_{12}$ are non-sparse for (a) to (c).
Also, by using $\acute{\lambda}_{\max}(\bSig_i)$, we estimated the largest eigenvalues as $1896$, $3206$ and $2101$ for ALL, MLL and AML, respectively. 
From this observation, we concluded that $\bSig_i$s are non-sparse. 
We estimated $\tr(\bSig_{\max}^2)/(n_{\min}\Delta_{(I)}^2)$ and $\lambda_{\max}/\Delta_{(I)}$ by $C_1=\max\{W_{1n_1},W_{2n_2}\}/(n_{\min}\hat{\Delta}_{(I)}^2)$ and $C_{2}=\max\{ \acute{\lambda}_{\max}(\bSig_1),\acute{\lambda}_{\max}(\bSig_2) \}/\hat{\Delta}_{(I)}$ in (C-i') and (C-ii'). 
Then, we had ($C_1,C_2$) as (0.362, 0.787) for (a), (0.001, 0.14) for (b), and (0.082, 0.375) for (c). 
Note that $\liminf_{p\to \infty}\Delta_{\min(II)}/\Delta_{(I)}>0$ and $\liminf_{p\to \infty}\Delta_{\min(III)}/\Delta_{(I)}>0$. 
From these observations, it is likely that the classifiers by (I) to (III) satisfy (C-i') and (C-ii') specially for (b) and hold the consistency property in (\ref{2.2}) from Proposition \ref{pro2.1}. 

Based on all the samples, we checked the accuracy of the classifiers by using the LOOCV for (a) to (c).
We checked the accuracy for 3-class classification as well by using the multiclass classification rule given in Remark \ref{rem1.1}. 
In the 3-class classification, we used $\hat{\theta}_j$ given in Remark \ref{rem5.1} for FS-DQDA and used the one-versus-one approach for HM-LSVM.
We summarized misclassification rates in Table~\ref{Tab3}. 
We observed that FS-DQDA gives excellent performances. 
HM-LSVM also gave reasonable performances, however, it does not draw information about the difference of the covariance matrices. 
See Section 2.2 in \cite{Aoshima:2014} for such an example. 
As for (b), all the classifiers gave preferable performances.
This is probably because the classifiers by (I) to (III) satisfy (C-i') and (C-ii') for (b). 
\begin{table}
\begin{center}
\caption{Error rates of the classifiers for samples from \cite{Armstrong:2002}.}
\label{Tab3}
{\footnotesize
\begin{tabular}{c|cccccc}
\hline
\\[-3mm]
 Classifier & DBDA & GQDA & DLDA-bc & DQDA-bc & FS-DQDA & HM-LSVM \\ 
\hline \\[-3mm]
&\multicolumn{6}{c}{LOOCV of samples from (a) ALL: 24 and MLL: 20}   \\[0mm]
Error rate & 1/44 & 2/44 & 6/44 & 1/44 & 0/44  & 0/44 \\[1mm]
\hline \\[-3mm]
&\multicolumn{6}{c}{LOOCV of samples from (b) ALL: 24 and AML: 28}   \\[0mm]
Error rate & 1/52 & 1/52 & 1/52 & 0/52 & 0/52  & 0/52 \\[1mm]
\hline \\[-3mm]
&\multicolumn{6}{c}{LOOCV of samples from (c) MLL: 20 and AML: 28}   \\[0mm]
Error rate & 4/48 & 4/48 & 1/48 & 3/48 & 3/48  & 3/48 \\[1mm]
\hline \\[-3mm]
&\multicolumn{6}{c}{LOOCV of samples from ALL: 24, MLL: 20 and AML: 28}   \\[0mm]
Error rate & 5/72   & 6/72 & 7/72 & 4/72 & 2/72& 3/72 \\[1mm]
\hline    
\end{tabular}
}
\end{center}
\end{table}
\section{Concluding remarks}
In this paper, we considered high-dimensional quadratic classifiers in non-sparse settings. 
The classifier based on the Mahalanobis distance does not always give a preferable performance even when $n_{\min}\to \infty$ and $\pi_i$s are assumed Gaussian, having known covariance matrices. 
See Sections 1 and 3. 
We emphasize that the quadratic classifiers proposed in this paper draw information about heterogeneity effectively through both the differences of mean vectors and covariance matrices. 
See Section 3.4 for the details. 
If the difference is not sufficiently large, we recommend to use the linear classifiers, DBDA and DLDA-bc (or the original DLDA).
They are quite flexible about the conditions to claim the consistency property. 
See Sections 4.2 and 4.3 for the details. 
We emphasize that DLDA-bc, DQDA-bc and FS-DQDA can hold the consistency property under at least $n_{\min}^{-1}\log{p}=o(1)$. 
Thus we do not recommend to use the classifiers when $n_{\min}^{-1}\log{p}\ne o(1)$. 
In such cases, one should use DBDA and GQDA because they hold the consistency property even when $n_i$s are fixed. 
See Section 4.2 about the choice between DBDA and GQDA. 
When $n_{\min}^{-1}\log{p}=o(1)$, we recommend DQDA-bc and FS-DQDA. 
Especially, FS-DQDA can claim the consistency property even when $n_{\min}/p\to 0$ and $\Delta_{\min}$ is not sufficiently large. 
See Section 5.1 for the details. 
For a choice of $\gamma\in (0,1)$ in (\ref{5.1}), one may apply cross-validation procedures or simply choose as $\gamma=0.5$. 
Actually, FS-DQDA with $\gamma=0.5$ gave preferable performances throughout our simulations and real data analyses. 
On the other hand, even when $n_{\min}^{-1}\log{p}=o(1)$, we do not recommend to use classifiers by the sparse estimation of $\bSig_i^{-1}$ unless (1) the eigenvalues are bounded in the sense that $\lambda(\bSig_i)\in (0,\infty)$ as $p\to \infty$, and (2) $\bSig_i$s are sparse in the sense that many elements of $\bSig_i$s are very small. 
We emphasize that ``$\lambda_{\max}(\bSig_i)$s are bounded" is a strict condition since the eigenvalues should depend on $p$ and it is probable that $\lambda_{ij}\to \infty$ as $p\to \infty$ for the first several $j$s. 
See \cite{Yata:2013b} for the details. 
Also, the computational cost of the classifiers by the sparse estimation is terribly high.

In conclusion, we hope we have given simpler classifiers which will be more effective in the real world analysis of high-dimensional data. 
\appendix
\section*{Appendix A}
\def\theequation{A.\arabic{equation}}
We give proofs of the theorems. 
For proofs of the corollaries and the propositions, see Appendix B.
\begin{proof}[Proof of Theorem 2.1] 
We consider the case when $\bx_0\in \pi_1$. 
Under (C-i) and (C-ii), it holds that for $i=1,2$ 
\begin{align}
&\Var \{(\bx_0-\bmu_1)^T \bA_i(\overline{\bx}_{i n_i}-\bmu_i) \}=\tr(\bSig_i\bA_i \bSig_1\bA_i)/n_i=o(\Delta_{1}^2)\notag \\
&\mbox{and} \ \Var \{(\bx_0-\bmu_1-\overline{\bx}_{2n_2}+\bmu_2)^T\bA_2\bmu_{12} \} =\bmu_{12}^T\bA_2(\bSig_1+\bSig_2/n_2)\bA_2\bmu_{12} =o(\Delta_{1}^2) \label{A.1}
\end{align}
from the fact that 
$$
\bmu_{12}^T \bA_2 \bSig_2 \bA_2 \bmu_{12} \le \bmu_{12}^T\bA_2 \bmu_{12} \lambda_{\max}(\bA_2^{1/2} \bSig_2\bA_2^{1/2} ) \le \Delta_{1}\tr\{(\bSig_2\bA_2)^2\}^{1/2}=o(n_2\Delta_{1}^2)
$$ 
under (C-i). 
Note that $(\overline{\bx}_{in_i}-\bmu_i)^T\bA_i(\overline{\bx}_{in_i}-\bmu_i)-\tr(\bA_i\bS_{in_i})/n_i=\sum_{k \neq k'}^{n_i}(\bx_{ik}-\bmu_i)^T\bA_i(\bx_{ik'}-\bmu_i)/\{n_i(n_i-1)\}$.
Then, under (C-i) it follows that for $i=1,2$
\begin{align}
\Var\{(\overline{\bx}_{in_i}-\bmu_i)^T\bA_i(\overline{\bx}_{in_i}-\bmu_i)-\tr(\bA_i\bS_{in_i})/n_i\}
=O[\tr\{(\bSig_i\bA_{i})^2\}/n_i^2]=o(\Delta_{1}^2).
\label{A.2} 
\end{align}
Then, by using Chebyshev's inequality, from (\ref{A.1}) and (\ref{A.2}), we find that 
\begin{align}
W_2(\bA_2)-W_1(\bA_1) =\tr[\{ (\bx_0-\bmu_1)(\bx_0-\bmu_1)^T-\bSig_1\}(\bA_2-\bA_1)]+\Delta_{1}+o_P(\Delta_{1}).
\label{A.3}
\end{align}
Here, under (A-i) and (C-iii), it follows that 
\begin{align}
\Var \big( \tr[\{ (\bx_0-\bmu_1)(\bx_0-\bmu_1)^T-\bSig_1\}(\bA_2-\bA_1)] \big)&=O\big(\tr[\{\bSig_1(\bA_2-\bA_1)\}^2] \big)=o(\Delta_{1}^2).
\label{A.4}
\end{align}
Thus by combining (\ref{A.3}) with (\ref{A.4}), under (A-i) and (C-i) to (C-iii), we obtain that $\{W_2(\bA_2)-W_1(\bA_1)\}/\Delta_{1}=1+o_P(1)$, 
so that $P\{W_2(\bA_2)-W_1(\bA_1)>0 \}\to 1$. 
When $\bx_0\in\pi_2$, we have the same arguments.
The proof is completed. 
\end{proof}
\begin{proof}[Proof of Theorem 3.1] 
Note that $\tr\{(\bSig_i\bA_{i})^2\}/n_i^2=o(\delta_{i}^2)$, $i=1,2$.
Then, similar to (\ref{A.1}) to (\ref{A.4}), under (A-i) and (C-iv) to (C-vi), we have that as $m\to \infty$
\begin{align}
W_{i'}(\bA_{i'})-W_i(\bA_i)-\Delta_{i}=&2(\bx_0-\bmu_i)^T\{\bA_i(\overline{\bx}_{in_i}-\bmu_i)-\bA_{i'}(\overline{\bx}_{i'n_{i'}}-\bmu_{i'})\}+o_P(\delta_{i})
\label{A.5}
\end{align}
when $\bx_0\in \pi_i$ $(i'\neq i)$. 
Note that $2\omega_{i}/\delta_{i}\to 1$ as $m\to \infty$ for $i=1,2,$ under (C-vi), where 
$\omega_{i}=\{\tr\{(\bSig_i\bA_i)^2\}/n_i+\tr(\bSig_i \bA_{i'} \bSig_{i'} \bA_{i'})/n_{i'}\}^{1/2}$ $(i' \neq i)$ in view of Lemma B.1 of Appendix B. 
Then, by combining Lemma B.1 with (\ref{A.5}), we conclude the results.
\end{proof}
\begin{proof}[Proof of Theorem 3.2] 
Similar to (\ref{A.5}), under (A-i), (C-iv) and (C-v), we have that as $m\to \infty$
\begin{align}
W_{i'}(\bA_{i'})-W_i(\bA_i)-\Delta_{i}=&2(\bx_0-\bmu_i)^T\{\bA_i(\overline{\bx}_{in_i}-\bmu_i) \notag \\
&-\bA_{i'}(\overline{\bx}_{i'n_{i'}}-\bmu_{i'}+(-1)^i\bmu_{12}) \}+o_P(\delta_{i})
\label{A.6}
\end{align}
when $\bx_0\in \pi_i$ $(i' \neq i)$. 
Then, by combining Lemma B.2 of Appendix B with (\ref{A.6}), we conclude the results.
\end{proof}
\begin{proof}[Proof of Theorem 5.1]
By using (B.23) and (B.24) in Appendix B, we claim the result. 
\end{proof}
\section*{Appendix B}
\def\theequation{B.\arabic{equation}}
\def\thesection{B}
Throughout, 
we consider the eigen-decomposition of $\bA_{i}$ by $\bA_{i}=\bH_{i(A)}\bLam_{i(A)}\bH_{i(A)}^T$ for $i=1,2$, where $\bLam_{i(A)}=$diag$(\lambda_{i1(A)},...,\lambda_{ip(A)})$ having eigenvalues such as $\lambda_{i1(A)}\ge \cdots \ge \lambda_{ip(A)}> 0$ and $\bH_{i(A)}=[\bh_{i1(A)},...,\bh_{ip(A)}]$ is an orthogonal matrix of the corresponding eigenvectors. 
Let $a_{i(j)}$ be the $j$-th diagonal element of $\bA_i$ for $j=1,...,p\ (i=1,2)$. 
Let $\tilde{\bx}_{1k}=\bA_1^{1/2} (\bx_{1k}-\bmu_1)$ and $\tilde{\bx}_{2k}=\bA_1^{-1/2} \bA_2(\bx_{2k}-\bmu_2)$ for $k=1,...,n_i$.
Let $\tilde{\bSig}_{1}=\bA_1^{1/2}\bSig_1 \bA_1^{1/2}$, 
$\tilde{\bSig}_{2}=\bA_1^{-1/2} \bA_2 \bSig_2\bA_2\bA_1^{-1/2}$, 
$\tilde{\bGamma}_1=[\tilde{\bgam}_{11},...,\tilde{\bgam}_{1q_1}]=\bA_1^{1/2}\bGamma_1$ and
$\tilde{\bGamma}_2=[\tilde{\bgam}_{21},...,\tilde{\bgam}_{2q_2}]=\bA_1^{-1/2}\bA_2 \bGamma_2$. 
Note that $\Var(\tilde{\bx}_{ij})=\tilde{\bGamma}_i\tilde{\bGamma}_i^T=\sum_{j=1}^{q_i}\tilde{\bgam}_{ij}\tilde{\bgam}_{ij}^T =\tilde{\bSig}_{i}$, $i=1,2$. 
Let $\hat{\bB}_i=\hat{\bA}_i-{\bA}_i$ for $i=1,2$. 
Let $x_{oijk}=x_{ijk}-\mu_{ij}$ for $j=1,...,p\ (i=1,2;\ k=1,...,n_i)$.

\begin{proof}[Proof of Proposition 1.1]
We can write that $\tr(\bA_i^{-1}\bA_{i'})=\sum_{j=1}^p \bh_{ij(A)}^T\bA_{i'}\bh_{ij(A)}$
$/\lambda_{ij(A)}$. 
Note that $\sum_{j=1}^p \bh_{ij(A)}^T\bA_{i'}\bh_{ij(A)}=\tr(\bA_{i'})$ and 
$\sum_{j=1}^t(\bh_{ij(A)}^T\bA_{i'}\bh_{ij(A)}-\lambda_{i'j(A)})\le 0$ for any $t \in \{1,...,p\}$. 
Then, by noting that $\lambda_{i1(A)}\ge \cdots \ge \lambda_{ip(A)}> 0$, 
we have that 
\begin{align}
\tr(\bA_i^{-1}\bA_{i'})
&=\frac{\lambda_{i'1(A)} }{\lambda_{i1(A)}}+
\frac{\bh_{i1(A)}^T\bA_{i'}\bh_{i1(A)}-\lambda_{i'1(A)} }{\lambda_{i1(A)}}+
\sum_{j=2}^p\frac{\bh_{ij(A)}^T\bA_{i'}\bh_{ij(A)}}{\lambda_{ij(A)}} \notag \\
&\ge 
\sum_{j=1}^2 \frac{\lambda_{i'j(A)} }{\lambda_{ij(A)}}+
\sum_{j=1}^2 \frac{\bh_{ij(A)}^T\bA_{i'}\bh_{ij(A)}-\lambda_{i'j(A)} }{\lambda_{i2(A)}}+
\sum_{j=3}^p\frac{\bh_{ij(A)}^T\bA_{i'}\bh_{ij(A)}}{\lambda_{ij(A)}}\notag \\
&\quad \vdots \notag \\
\label{B.1}
&\ge 
\sum_{j=1}^p \frac{\lambda_{i'j(A)} }{\lambda_{ij(A)}}+
\sum_{j=1}^p \frac{\bh_{ij(A)}^T\bA_{i'}\bh_{ij(A)}-\lambda_{i'j(A)} }{\lambda_{ip(A)}}
=\sum_{j=1}^p \frac{\lambda_{i'j(A)} }{\lambda_{ij(A)}}.
\end{align}
Thus, when $\tr\{\bSig_i (\bA_{i'}-\bA_i) \}=\tr(\bA_i^{-1}\bA_{i'})-p$, 
it holds that 
$$
\Delta_{i} \ge \sum_{j=1}^p\{\lambda_{i'j(A)}/\lambda_{ij(A)}-1+\log(\lambda_{ij(A)}/\lambda_{i'j(A)}) \}\ge 0 
$$ 
from the fact that $c-1+\log{c^{-1}}\ge 0$ for any positive constant $c$. 
Note that $\lambda_{1j(A)}\neq \lambda_{2j(A)}$ or $\bh_{ij(A)}^T\bA_{i'}\bh_{ij(A)}<\lambda_{i'j(A)}$ for some $j$ 
when $\bA_1 \neq \bA_2$. 
Since $c-1+\log{c^{-1}}> 0$ when $c \neq 1\ $, it holds that $\Delta_{i} >0$ when $\lambda_{1j(A)}\neq \lambda_{2j(A)}$ for some $j$. 
From (\ref{B.1}), if $\bh_{ij(A)}^T\bA_{i'}\bh_{ij(A)}<\lambda_{i'j(A)}$ for some $j$, 
it follows that $\tr(\bA_i^{-1}\bA_{i'})>\sum_{j=1}^p(\lambda_{i'j(A)}/\lambda_{ij(A)})$, 
so that $\Delta_{i} >0$. 
When $\bmu_1 \neq \bmu_2$, it holds that $\Delta_{i}\ge \bmu_{12}^T\bA_{i'} \bmu_{12}>0$. 
Hence, it concludes the results. 
\end{proof}
\begin{proof}[Proof of Proposition 2.1]
We note that
\begin{align}
\label{B.2}
&\Delta_{iA} \le 
\bmu_{12}^T\bA_{i'}\bmu_{12}\lambda_{\max}(\bA_{i'}^{1/2} \bSig_i \bA_{i'}^{1/2})\le 
\Delta_{i}\lambda_{\max}(\bA_{i'}^{1/2} \bSig_i \bA_{i'}^{1/2}) \\
&\notag
\mbox{and }\ \tr(\bSig_i \bA_{i'}\bSig_{i'} \bA_{i'}) \le \tr\{(\bSig_i \bA_{i'})^2\}^{1/2} \tr\{(\bSig_{i'} \bA_{i'})^2\}^{1/2}.
\end{align}
When $\limsup_{p\to \infty}\lambda_{i1(A)}<\infty$, $i=1,2$, it holds that 
\begin{align}
\label{B.3}
&\lambda_{\max}(\bA_{i'}^{1/2} \bSig_i \bA_{i'}^{1/2})\le \lambda_{i1}\lambda_{\max}(\bA_{i'})=\lambda_{i1}\lambda_{i'1(A)}=O(\lambda_{i1}) \ \ \mbox{and} \\
&\tr\{(\bSig_{l}\bA_{l'})^2\}\le \tr(\bSig_{l}\bA_{l'}\bSig_{l})\lambda_{l'1(A)}\le \tr(\bSig_{l}^2)\lambda_{l'1(A)}^{2}=O\{ \tr(\bSig_{l}^2)\} \notag
\end{align}
for all $l,l'$.
By combining (\ref{B.2}) with (\ref{B.3}), (C-i') and (C-ii') imply (C-i) and (C-ii). 

Next, for (C-iii), it holds that 
$\tr[\{\bSig_i(\bA_1-\bA_2)\}^2]\le \lambda_{i1}\tr\{(\bA_1-\bA_2)\bSig_i(\bA_1-\bA_2) \}$. 
When $\bA_i$s are diagonal matrices such as $\bA_i=\mbox{diag}(a_{i(1)},....,a_{i(p)}),\ i=1,2$, 
it holds that 
$\Delta_{i}\ge \sum_{j=1}^p\{a_{i'(j)}/a_{i(j)}-1-\log(a_{i'(j)}/a_{i(j)})\}$ and
$\tr\{(\bA_1-\bA_2)\bSig_i(\bA_1-\bA_2)\}=\sum_{j=1}^p\sigma_{i(j)}(a_{1(j)}-a_{2(j)})^2$. 
Note that $a_{i(j)}\in (0,\infty)$ as $p\to \infty$ for all $i,j$, under $\lambda(\bA_i)\in (0,\infty)$ as $p\to \infty$ for $i=1,2$. 
By Taylor expansion, we claim that 
$$
a_{i'(j)}/a_{i(j)}-1-\log(a_{i'(j)}/a_{i(j)})\ge a_{i(j)}^{-2}(a_{1(j)}-a_{2(j)})^2/(2\max \{1, a_{i'(j)}^2/a_{i(j)}^2\}).
$$
Then, it follows that $\sum_{j=1}^p\sigma_{i(j)}(a_{1(j)}-a_{2(j)})^2=O(\Delta_{i})$ because $\sigma_{i(j)}\in (0,\infty)$ as $p\to \infty$ for all $i,j$. 
Thus we have that $\tr[\{\bSig_i(\bA_1-\bA_2)\}^2]=O(\Delta_{i}\lambda_{i1})$.
It concludes the results. 
\end{proof}
\begin{proof}[Proofs of Corollaries 2.1 and 2.2]
From Theorem 2.1 and Proposition 2.1, we can claim Corollaries 2.1 and 2.2 straightforwardly. 
\end{proof}
\begin{proof}[Proof of Proposition 2.2] 
We first consider the case when $\liminf_{p\to \infty}\sum_{j=1}^p|\lambda_{ij}/\lambda_{i'j}-1|/p>0$. 
When $c_{1j}< | \lambda_{ij}/\lambda_{i'j}-1|<c_{2j}$ for some constants $c_{1j}\ (>0)$ and $c_{2j}\ (<\infty)$, 
by Taylor expansion, it holds that 
$$
\lambda_{ij}/\lambda_{i'j}-1-\log(\lambda_{ij}/\lambda_{i'j})\ge \frac{(\lambda_{ij}/\lambda_{i'j}-1)^2}{2\max \{1, \lambda_{ij}^2/\lambda_{i'j}^2\}}\ge \frac{c_{1j}|\lambda_{ij}/\lambda_{i'j}-1|}{ 2(c_{2j}+1)^2}.
$$ 
When $ \lambda_{ij}/\lambda_{i'j} \to \infty$ as $p\to \infty$, it holds that for sufficiently large $p$ 
$$
\lambda_{ij}/\lambda_{i'j}-1-\log(\lambda_{ij}/\lambda_{i'j})>|\lambda_{ij}/\lambda_{i'j}-1|/2.
$$ 
Thus, when $\liminf_{p\to \infty}\sum_{j=1}^p|\lambda_{ij}/\lambda_{i'j}-1|/p>0$, it follows that $\liminf_{p\to \infty}\Delta_{i(IV)}/p \ge \liminf_{p\to \infty} \sum_{j=1}^p\{\lambda_{ij}/\lambda_{i'j}-1-\log(\lambda_{ij}/\lambda_{i'j})\}/p>0$ from (\ref{B.1}). 

Next, we consider the case when $\liminf_{p\to \infty} |\tr(\bSig_i\bSig_{i'}^{-1})/p-1 |>0$. 
We note that $\tr(\bSig_i\bSig_{i'}^{-1})\ge \sum_{j=1}^p\lambda_{ij}/\lambda_{i'j} $ from (\ref{B.1}). 
When $\tr(\bSig_i\bSig_{i'}^{-1})/(\sum_{j=1}^p\lambda_{ij}/\lambda_{i'j})\to 1$ as $p\to \infty$, 
it holds that $\liminf_{p\to \infty}|\sum_{j=1}^p(\lambda_{ij}/\lambda_{i'j})/p-1|>0$ 
under $\liminf_{p\to \infty} |\tr(\bSig_i\bSig_{i'}^{-1})/p-1 |>0$.
It follows that $\liminf_{p\to \infty} \Delta_{i(IV)}/p>0$ from the fact that $\sum_{j=1}^p|\lambda_{ij}/\lambda_{i'j}-1|/p \ge |\sum_{j=1}^p(\lambda_{ij}/\lambda_{i'j})/p-1|$. 
On the other hand, we note that 
$$
\Delta_{i(IV)} \ge \tr(\bSig_i\bSig_{i'}^{-1})-p-\sum_{j=1}^p\log(\lambda_{ij}/\lambda_{i'j})\ge \tr(\bSig_i\bSig_{i'}^{-1})-\sum_{j=1}^p(\lambda_{ij}/\lambda_{i'j})
$$ 
because $\sum_{j=1}^p \{\lambda_{ij}/\lambda_{i'j}-1-\log(\lambda_{ij}/\lambda_{i'j})\}\ge 0$. 
Thus, when $\sum_{j=1}^p(\lambda_{ij}/\lambda_{i'j})/p-1\to 0$ as $p\to \infty$ and $\liminf_{p\to \infty}\{\tr(\bSig_i\bSig_{i'}^{-1})/(\sum_{j=1}^p\lambda_{ij}/\lambda_{i'j})\}> 1$, we have that $\liminf_{p\to \infty}$ $\Delta_{i(IV)}/p>0$. 
Hence, it concludes the results. 
\end{proof}

\begin{proof}[Proof of Proposition 3.1] 
Under $\limsup_{p\to \infty} \lambda_{i1(A)}<\infty$ for $i=1,2$, 
we have that $\tr\{(\bSig_{i'}\bA_{i'})^2\}=O\{\tr(\bSig_{i'}^2)\}$ and 
\begin{align*} 
&\tr\{(\bSig_i \bA_l \bSig_l \bA_l )^2\}=\tr\{(\bSig_i^{1/2} \bA_l\bSig_l \bA_l \bSig_i^{1/2})^2\}\\
&\le \lambda_{\max}(\bSig_i^{1/2} \bA_l \bSig_l \bA_l \bSig_i^{1/2})\tr(\bSig_i^{1/2} \bA_l \bSig_l \bA_l \bSig_i^{1/2})\\
&\le \lambda_{\max}(\bSig_i^{1/2} \bA_l^{2} \bSig_i^{1/2} )\lambda_{l1} \delta_{i}^2n_l=O(\lambda_{i1}\lambda_{l1}\delta_{i}^2n_l); \ \ \mbox{and}  \\
&\bmu_{12}^T\bA_{i'} \bSig_l \bA_{i'} \bmu_{12}\le ||\bmu_{12}||^2\lambda_{\max}(\bA_{i'} \bSig_l \bA_{i'} )=O(||\bmu_{12}||^2 \lambda_{l1})\ \ \mbox{for $l=i,i'$.}
\end{align*}
Then, when $\limsup_{p\to \infty} \lambda_{i1(A)}<\infty$, $i=1,2$, 
(C-iv') and (C-vi') imply (C-iv) and (C-vi), respectively. 
Similar to Proof of Proposition 2.1, we can claim the result for (C-v') from $\tr\{(\bA_1-\bA_2)^2\}=\sum_{j=1}^p (a_{1(j)}-a_{2(j)})^2$ 
when $\bA_i$s are diagonal matrices. 
The proof is completed. 
\end{proof}

\begin{lem}
Let $\omega_{i}=\{\tr\{(\bSig_i \bA_i)^2\}/n_i+\tr(\bSig_i \bA_{i'}\bSig_{i'} \bA_{i'})/n_{i'}\}^{1/2}$ for $i=1,2\ (i'\neq i)$. 
Then, under (A-i), (C-iv) and (C-vi), we have that 
$$
(\bx_0-\bmu_i)^T\{\bA_i (\overline{\bx}_{in_i}-\bmu_i)-\bA_{i'}(\overline{\bx}_{i'n_{i'}}-\bmu_{i'}) \}/\omega_{i}\Rightarrow N(0,1) \ \ \mbox{as $m\to \infty$}
$$
when $\bx_0 \in \pi_i$ for $i=1,2\ (i'\neq i)$. 
\end{lem}
\begin{proof}[Proof of Lemma B.1] 
We consider the case when $i=1$ $(i'=2)$ and $\bx_0 \in \pi_1$. 
Let $\tilde{\bx}_{0}=\bA_1^{1/2}(\bx_0-\bmu_1)$. 
Then, it holds that $\Var(\tilde{\bx}_{0}|\bx_0\in \pi_1)=\Var(\tilde{\bx}_{1k})=\tilde{\bSig}_{1}$. 
Let 
$$
v_{k}=\tilde{\bx}_0^T\tilde{\bx}_{1k}/(n_1 \omega_{1}),\ k=1,...,n_1,\ \mbox{and}\ 
v_{n_1+k}=-\tilde{\bx}_0^T \tilde{\bx}_{2k}/(n_2\omega_{1}),\ k=1,...,n_2.
$$
Note that $\sum_{k=1}^{n_1+n_2} E(v_{k}^2)=1$ and 
$$
\sum_{k=1}^{n_1+n_2}v_{k}=(\bx_0-\bmu_1)^T\{\bA_1(\overline{\bx}_{1n_1}-\bmu_1)-\bA_2(\overline{\bx}_{2n_2}-\bmu_2) \}/\omega_{1}.
$$
Then, it holds that $E(v_{k}|v_{k-1},...,v_{1})=0$ for $k=2,...,n_1+n_2$. 
We consider applying the martingale central limit theorem given by \cite{McLeish:1974}. 
Under (A-i), we can write that 
$
\tilde{\bx}_{1l}=\tilde{\bGamma}_1\by_{1l}
$
and 
$
\tilde{\bx}_{2l}=\tilde{\bGamma}_2\by_{2l}.
$
Then, in a way similar to the equations (23) and (24) in \cite{Aoshima:2014}, we can evaluate that under (A-i)
\begin{align}
& (n_s \omega_{1})^4E(v_{k}^4)=3\tr(\tilde{\bSig}_1\tilde{\bSig}_s)^2+O[\tr \{(\tilde{\bSig}_1\tilde{\bSig}_s)^2 \}] \quad\mbox{and}
\label{B.4} \\
&(n_sn_{s'})^2 \omega_{1}^4 E(v_{k}^2v_{k'}^2)=\tr(\tilde{\bSig}_1\tilde{\bSig}_s)\tr(\tilde{\bSig}_1\tilde{\bSig}_{s'})
+2\tr(\tilde{\bSig}_1\tilde{\bSig}_s \tilde{\bSig}_1\tilde{\bSig}_{s'})\notag \\
&\hspace{3.5cm} +O[\{\tr(\tilde{\bSig}_1\tilde{\bSig}_s \tilde{\bSig}_1\tilde{\bSig}_{s})
\tr(\tilde{\bSig}_1\tilde{\bSig}_{s'} \tilde{\bSig}_1\tilde{\bSig}_{s'})\}^{1/2}]\label{B.5} 
\end{align}
for $k\neq k'$, 
where $s=1$ for $k \in [1,...,n_1]$, $s=2$ for $k\in [n_1+1,...,n_1+n_2]$, $s'=1$ for $k'\in [1,...,n_1]$, and $s'=2$ for $k'\in [n_1+1,...,n_1+n_2]$. 
Note that $\tr(\tilde{\bSigma}_1^4) \le \tr(\tilde{\bSigma}_1^2)^2$ and $\tr\{(\tilde{\bSigma}_1\tilde{\bSigma}_2)^2 \} \le \tr(\tilde{\bSigma}_1\tilde{\bSigma}_2)^2$. 
Then, by using Chebyshev's inequality and Schwarz's inequality, from (\ref{B.4}), under (A-i), we have that for Lindeberg's condition 
\begin{align}
\sum_{k=1}^{n_1+n_2} E \{  v_{k}^2 I ( 
 v_{k}^2 \ge \tau  ) \}\le \sum_{k=1}^{n_1+n_2} \frac{E(v_{k}^4)}{\tau}=O\Big[\frac{\tr(\tilde{\bSigma}_1^2)^2/n_1^3 +\tr(\tilde{\bSigma}_1\tilde{\bSigma}_2)^2/n_2^3}{
 \omega_{1}^4 } \Big] \to 0 
 \notag
\end{align}
as $m\to \infty$ for any $\tau>0$, where $I(\cdot)$ is the indicator function. 
Since $2\omega_{1}/\delta_{1}=1+o(1)$ under (C-vi), we note that  
\begin{align*}
&\frac{\tr(\tilde{\bSigma}_1^4)}{n_1^2 \omega_{1}^4 }\to 0, \quad  \frac{ \tr\{(\tilde{\bSigma}_1\tilde{\bSigma}_2)^2\}}{n_2^2\omega_{1}^4}
=\frac{\tr\{(\bSigma_1\bA_2 \bSigma_2\bA_2)^2\}}{n_2^2 \omega_{1}^4}\to 0,\\
&\mbox{and}\quad \frac{\tr(\tilde{\bSigma}_1^3\tilde{\bSigma}_2)}{n_1 n_2 \omega_{1}^4}\le 
\frac{\tr(\tilde{\bSigma}_1^4)^{1/2}\tr\{(\tilde{\bSigma}_1\tilde{\bSigma}_2 )^2\}^{1/2}}{n_1 n_2\omega_{1}^4}
\to 0
\end{align*}
under (C-iv).  
Then, by using Chebyshev's inequality, from (\ref{B.4}) and (\ref{B.5}), under (A-i), (C-iv) and (C-vi), we have that for any $\tau>0$ 
\begin{align*}
&P\Big(\Big|\sum_{k=1}^{n_1+n_2} v_{k}^2-1\Big|\ge \tau \Big) \notag \\
&=
O\Big[\frac{\tr(\tilde{\bSigma}_1^4)/n_1^2+\tr(\tilde{\bSigma}_1^4)^{1/2}\tr\{(\tilde{\bSigma}_1\tilde{\bSigma}_2)^2 \}^{1/2}/(n_1n_2)+
\tr\{(\tilde{\bSigma}_1\tilde{\bSigma}_2)^2 \}/n_2^2}{ \omega_{1}^4 } \Big]+o(1)\to 0
\end{align*}
as $m\to \infty$, 
so that $ \sum_{k=1}^{n_1+n_2} v_{k}^2=1+o_P(1)$. 
Hence, by using the martingale central limit theorem, we obtain that $\sum_{k=1}^{n_1+n_2}v_k \Rightarrow N(0,1)$ as $m\to \infty$ under (A-i), (C-iv) and (C-vi). 
Hence, we conclude the result when $i=1$. 
For the case when $i=2$, we can have the same arguments.
The proof is completed. 
\end{proof}
\begin{lem}
Under (A-ii), (C-iv) and (C-vii), we have that  
$$
2(\bx_0-\bmu_i)^T\big\{\bA_i(\overline{\bx}_{in_i}-\bmu_i)-\bA_{i'}(\overline{\bx}_{i'n_{i'}}-\bmu_{i'}+(-1)^i\bmu_{12}) \big\}
/\delta_{i}\Rightarrow N(0,1)
$$
as $m\to \infty$ when $\bx_0 \in \pi_i$ for $i=1,2\ (i'\neq i)$. 
\end{lem}
\begin{proof}[Proof of Lemma B.2] 
We consider the case when $i=1$ $(i'=2)$ and $\bx_0 \in \pi_1$. 
Let $\bx_0-\bmu_1=\bGamma_1\by_{0}$ and $\by_{0}=(y_{01},...,y_{0q_1})^T$. 
Under (A-ii), $y_{0s},\ s=1,...,q_1,$ are independent. 
Let $\acute{\bx}_{ln_l}=\sum_{k=1}^{n_{l}} \tilde{\bx}_{lk}/n_{l}$, $l=1,2$, $\tilde{\bmu}=\bA_1^{-1/2}\bA_2 \bmu_{12}$ and 
$$
w_s=2y_{0j} \tilde{\bgam}_{1s}^T(\acute{\bx}_{1n_1}-\acute{\bx}_{2n_2}+\tilde{\bmu})/\delta_{1},\ \ s=1,...,q_1.
$$
Note that $q_1\ge p$, $E(w_{s})=0$, $s=1,...,q_1,$  $\sum_{s=1}^{q_1} E(w_{s}^2)=1$ and 
$$
\sum_{s=1}^{q_1}w_{s}=
2(\bx_0-\bmu_1)^T\{\bA_1(\overline{\bx}_{1n_1}-\bmu_1)-\bA_2(\overline{\bx}_{2n_2}-\bmu_2-\bmu_{12}) \}/\delta_{1}.
$$
Also, note that $E(w_{s}|w_{s-1},...,w_{1})=0$ for $s=2,...,q_1$, under (A-ii). 
We consider applying the martingale central limit theorem. 
Let $M_{ls}=E(y_{lsk}^3)$ for all $l,s$. 
Note that $\limsup_{p\to \infty} |M_{ls}|<\infty$ for all $l,s$, under (A-ii) because $\limsup_{p\to \infty} E(y_{lsk}^4)<\infty$. 
Then, by using Schwarz's inequality and the arithmetic mean-geometric mean inequality, we can evaluate that under (A-ii) 
\begin{align*}
&E\{(\tilde{\bgam}_{1s}^T\acute{\bx}_{ln_l})^2(\tilde{\bgam}_{1t}^T\acute{\bx}_{ln_l})^2 \}
=\{1+o(1)\}\tilde{\bgam}_{1s}^T \tilde{\bSig}_{l} \tilde{\bgam}_{1s}\tilde{\bgam}_{1t}^T \tilde{\bSig}_{l} \tilde{\bgam}_{1t}/n_l^2 +O\{( \tilde{\bgam}_{1s}^T\tilde{\bSig}_{l}\tilde{\bgam}_{1t}/n_l)^2\};\ \mbox{and} \\
&|E\{(\tilde{\bgam}_{1s}^T \acute{\bx}_{ln_l})^2\tilde{\bgam}_{1t}^T\acute{\bx}_{ln_l} \tilde{\bgam}_{1t}^T\tilde{\bmu}\}| =\Big| \sum_{u=1}^{q_l}(\tilde{\bgam}_{1s}^T\tilde{\bgam}_{lu})^2
\tilde{\bgam}_{1t}^T\tilde{\bgam}_{lu}\tilde{\bgam}_{1t}^T\tilde{\bmu}M_{lu}/n_l^2\Big|\\
&\le \{ E(\tilde{\bgam}_{1s}^T\acute{\bx}_{ln_l})^4\}^{1/2}
\{ E(\tilde{\bgam}_{1t}^T\acute{\bx}_{ln_l} \tilde{\bgam}_{1t}^T\tilde{\bmu}  )^2\}^{1/2} \\
&=O\{\tilde{\bgam}_{1s}^T\tilde{\bSig}_{l} \tilde{\bgam}_{1s}(\tilde{\bgam}_{1t}^T \tilde{\bSig}_{l}\tilde{\bgam}_{1t}/n_l)^{1/2}|\tilde{\bgam}_{1t}^T\tilde{\bmu}|/n_l \} \\
&=O[\tilde{\bgam}_{1s}^T \tilde{\bSig}_{l} \tilde{\bgam}_{1s}\{ \tilde{\bgam}_{1t}^T \tilde{\bSig}_{l} \tilde{\bgam}_{1t}/n_l+(\tilde{\bgam}_{1t}^T\tilde{\bmu})^2 \}/n_l ] \\
&=O[ \{\tilde{\bgam}_{1s}^T \tilde{\bSig}_{l} \tilde{\bgam}_{1s}/n_l\}^2+\{\tilde{\bgam}_{1t}^T \tilde{\bSig}_{l} \tilde{\bgam}_{1t}/n_l\}^2+(\tilde{\bgam}_{1t}^T\tilde{\bmu})^4 ],\ l=1,2 
\end{align*}
for all $s,t$. 
Then, we have that for all $s,t$
\begin{align}
\label{B.6}
& \delta_{1}^4 E(w_{s}^4)=O\Big[ \sum_{l=1}^2\{\tilde{\bgam}_{1s}^T \tilde{\bSig}_{l} \tilde{\bgam}_{1s}/n_l\}^2+(\tilde{\bgam}_{1s}^T\tilde{\bmu})^4\Big]\quad\mbox{and}  \\
&(\delta_{1}/2)^4 \frac{E(w_{s}^2w_{t}^2)}{E(y_{0s}^2y_{0t}^2)}-
\tilde{\bgam}_{1s}^T\Big(\sum_{l=1}^2 \tilde{\bSig}_{l}/n_l+\tilde{\bmu}\tilde{\bmu}^T\Big)\tilde{\bgam}_{1s}
\tilde{\bgam}_{1t}^T\Big(\sum_{l=1}^2 \tilde{\bSig}_{l}/n_l+\tilde{\bmu}\tilde{\bmu}^T\Big)\tilde{\bgam}_{1t}\ 
\notag \\ 
&= 2\sum_{l=1}^2(-1)^{l+1} \sum_{u=1}^{q_l}  \{(\tilde{\bgam}_{1s}^T\tilde{\bgam}_{lu})^2
\tilde{\bgam}_{1t}^T\tilde{\bgam}_{lu}\tilde{\bgam}_{1t}^T+(\tilde{\bgam}_{1t}^T\tilde{\bgam}_{lu})^2
\tilde{\bgam}_{1s}^T\tilde{\bgam}_{lu}\tilde{\bgam}_{1s}^T\} \tilde{\bmu}M_{lu}/n_l^2 \notag \\ 
& \ \ +o\Big[ \sum_{l=1}^2 \tilde{\bgam}_{1s}^T \tilde{\bSig}_{l} \tilde{\bgam}_{1s} \tilde{\bgam}_{1t}^T \tilde{\bSig}_{l} \tilde{\bgam}_{1t}/n_l^2 \Big]+O\Big[ \sum_{l=1}^2 ( \tilde{\bgam}_{1s}^T \tilde{\bSig}_{l} \tilde{\bgam}_{1t}/n_l)^2 \Big]. \label{B.7}
\end{align}
Here, 
under (C-iv), we can evaluate that 
\begin{align}
\notag
&\sum_{s,t=1}^{q_1}\sum_{u=1}^{q_l}(\tilde{\bgam}_{1s}^T\tilde{\bgam}_{lu})^2
\tilde{\bgam}_{1t}^T\tilde{\bgam}_{lu} \tilde{\bgam}_{1t}^T\tilde{\bmu}M_{lu}/n_l^2
=\sum_{u=1}^{q_l} \tilde{\bgam}_{lu}^T\tilde{\bSig}_{1}\tilde{\bgam}_{lu}\tilde{\bgam}_{lu}^T \tilde{\bSig}_{1} \tilde{\bmu}M_{lu}/n_l^2\\
\notag
&=O\Big[||\tilde{\bmu}^T\tilde{\bSig}_{1}^{1/2}||
\sum_{u=1}^{q_l}||\tilde{\bgam}_{lu}^T\tilde{\bSig}_{1}^{1/2}|| \tilde{\bgam}_{lu}^T\tilde{\bSig}_{1}\tilde{\bgam}_{lu}/n_l^2
\Big] \\
\notag
&=O\Big[||\tilde{\bmu}^T\tilde{\bSig}_{1}^{1/2}||\tr(\tilde{\bSig}_{1}\tilde{\bSig}_{l})^{1/2}
\Big\{\sum_{u=1}^{q_l}(\tilde{\bgam}_{lu}^T\tilde{\bSig}_{1}\tilde{\bgam}_{lu})^2\Big\}^{1/2}/n_l^2 \Big] \\
&=O\big[\{\tilde{\bmu}^T\tilde{\bSig}_{1}\tilde{\bmu}+ \tr(\tilde{\bSig}_{1}\tilde{\bSig}_{l}) \}
\tr\{(\tilde{\bSig}_{1}\tilde{\bSig}_{l})^{2}\}^{1/2}/n_l^2 \Big]
=o(\delta_{1}^4),\quad l=1,2 \label{B.8}
\end{align}
from the fact that $\sum_{u=1}^{q_l}( \tilde{\bgam}_{lu}^T\tilde{\bSig}_{1} \tilde{\bgam}_{lu})^2
\le \sum_{u,u'=1}^{q_l}( \tilde{\bgam}_{lu}^T\tilde{\bSig}_{1} \tilde{\bgam}_{lu'})^2=\tr\{(\tilde{\bSig}_{1} \tilde{\bSig}_{l})^2\}$
$=o(n_l^2 \delta_{1}^4)$ under (C-iv). 
Then, by combining (\ref{B.6}) and (\ref{B.7}) with (\ref{B.8}), under (A-ii), (C-iv) and (C-vii), for any $\tau>0$, we have that as $m\to \infty$
\begin{align*}
&\sum_{s=1}^{q_1} \frac{E(w_{s}^4)}{\tau}
=O\Big[\frac{  
\sum_{l=1}^2 \tr\{(\tilde{\bSig}_{1} \tilde{\bSig}_{l})^2\}/n_l^2+
\sum_{s=1}^{q_1}(\tilde{\bgam}_{1s}^T\tilde{\bmu})^4}{ \delta_{1}^4 } \Big] \to 0 \ \ \mbox{and}\\
&P\Big(\Big|\sum_{s=1}^{q_1} w_{s}^2-1\Big|\ge \tau \Big)\le \frac{\sum_{s,t=1}^{q_1} E(w_{s}^2w_{t}^2)-1}{\tau^{2}}
=O \Big[ \sum_{s=1}^{q_1} E(w_{s}^4) \Big] +o(1)
\to 0,
\end{align*}
so that $\sum_{s=1}^{q_1} E \{  w_{s}^2 I (w_{s}^2 \ge \tau  ) \}\le  \sum_{s=1}^{q_1} E(w_{s}^4)/\tau \to 0$ and $\sum_{s=1}^{q_1}w_s^2=1+o_P(1)$. 
Hence, by using the martingale central limit theorem, we obtain that $\sum_{s=1}^{q_1} w_{s} \Rightarrow N(0,1)$ as $m\to \infty$ under (A-ii), (C-iv) and (C-vii). 
We conclude the result when $i=1$. 
For the case when $i=2$, we can have the same arguments.
The proof is completed. 
\end{proof}
\begin{proof}[Proofs of Corollaries 3.1 and 3.2]
From Theorems 3.1 and 3.2 and Proposition 3.1, we can claim Corollaries 3.1 and 3.2 straightforwardly. 
\end{proof}

\begin{lem}
Assume that when $\bx_0\in \pi_i$ for $i=1,2$ 
\begin{align}
&\tr[\{(\bx_0-\bmu_i)(\bx_0-\bmu_i)^T-\bSig_i\}(\hat{\bB}_1-\hat{\bB}_2)]=o_P(\kappa);  \label{B.9} \\
&\tr\{\bSig_i(\hat{\bB}_1-\hat{\bB}_2)\}-\log|\hat{\bA}_1 {\bA}_1^{-1} |+\log|\hat{\bA}_2 {\bA}_2^{-1} |=o_P(\kappa); \ \mbox{and} \label{B.10} \\
&\{2(\bx_0-\bmu_i)+(-1)^{i+1}\bmu_{12} \}^T\hat{\bB}_{i'} \bmu_{12}=o_P(\kappa)\ (i' \neq i)  \label{B.11} \\
&\mbox{and} \ \ (p/n_l^{1/2})|| \hat{\bB}_l||=o_P(\kappa),\ \ l=1,2,\notag
\end{align}
where $\kappa=\Delta_{\min}$ or $\kappa=\delta_{\min}$. 
Then, (4.1) holds.
\end{lem}
\begin{proof}[Proof of Lemma B.3] 
We consider the case when $\bx_0\in \pi_1$.
We have that  
\begin{align*}
&W_1(\hat{\bA}_1)-W_1(\bA_1)-W_2(\hat{\bA}_2)+W_2(\bA_2) \\
&=\tr[\{ (\bx_0-\bmu_1)(\bx_0-\bmu_1)^T-\bSig_1\}(\hat{\bB}_1-\hat{\bB}_2)] \\
&\quad +\tr\{\bSig_1(\hat{\bB}_1-\hat{\bB}_2)\}-\log|\hat{\bA}_1 {\bA}_1^{-1} |+\log|\hat{\bA}_2 {\bA}_2^{-1} | \\
&\quad +\sum_{l=1}^2(-1)^{l+1}
\tr[\{2(\bx_0-\bmu_1-(\overline{\bx}_{ln_l}-\bmu_1)/2)(\bmu_1-\overline{\bx}_{ln_l})^T-\bS_{ln_l}/n_l\}\hat{\bB}_l].
\end{align*}
Note that $\tr(\bS_{ln_l})=O_P(p)$, 
$||\overline{\bx}_{ln_l}-\bmu_{1}||^2\le ||\overline{\bx}_{ln_l}-\bmu_l||^2+||\bmu_{l}-\bmu_{1}||^2=||\bmu_l-\bmu_{1}||^2+O_P(p/n_l)$ and 
$|| \bx_0-\bmu_1-(\overline{\bx}_{ln_l}-\bmu_1)/2||^2\le || \bx_0-\bmu_1||^2+||\overline{\bx}_{ln_l}-\bmu_l||^2+||\bmu_1-\bmu_l||^2=O_P(p)$, $l=1,2$, from the facts that 
$E(|| \bx_0-\bmu_1||^2)=\tr(\bSig_1)$,
$E\{\tr(\bS_{ln_l}) \}=\tr(\bSig_l)$, $E(||\overline{\bx}_{ln_l}-\bmu_l||^2)=\tr(\bSig_l)/n_l$, 
$\tr(\bSig_i)=O(p)$, $i=1,2,$ and $||\bmu_{12}||^2=O(p)$. 
Then, we have that for $l=1,2$ 
\begin{align*}
&|\tr[\{2(\bx_0-\bmu_1-(\overline{\bx}_{ln_l}-\bmu_1)/2)(\bmu_l-\overline{\bx}_{ln_l})^T-\bS_{ln_l}/n_l\}\hat{\bB}_l]|\\
& \le 2|| \bx_0-\bmu_1-(\overline{\bx}_{ln_l}-\bmu_1)/2||\cdot ||\overline{\bx}_{ln_l}-\bmu_l||\cdot ||\hat{\bB}_l ||+\tr(\bS_{ln_l})||\hat{\bB}_l||/n_l\\
&=O_P\{ (p/n_l^{1/2}) ||\hat{\bB}_l ||\}.
\end{align*}
Also, we have that 
$|(\overline{\bx}_{2n_2}-\bmu_2)^T \hat{\bB}_2 \bmu_{12}|=O_P\{ (p/n_2^{1/2}) ||\hat{\bB}_2 ||$. 
Thus it holds that 
\begin{align*}
&\sum_{l=1}^2(-1)^{l+1}\tr[\{2(\bx_0-\bmu_1-(\overline{\bx}_{ln_l}-\bmu_1)/2)(\bmu_1-\overline{\bx}_{ln_l})^T-\bS_{ln_l}/n_l\}\hat{\bB}_l]\\
&=-\{2(\bx_0-\bmu_1)+\bmu_{12} \}^T\hat{\bB}_2 \bmu_{12}+O_P\{(p/n_1^{1/2}) ||\hat{\bB}_1 ||+(p/n_2^{1/2}) ||\hat{\bB}_2||\}.
\end{align*}
Hence, it concludes the result when $\bx_0\in \pi_1$. 
For the case when $\bx_0\in \pi_2$, we can have the same arguments.
The proof is completed. 
\end{proof}
\begin{proof}[Proofs of Propositions 4.1 and 4.2]
We consider the case when $\bx_0\in \pi_1$. 
Similar to Proof of Lemma B.3, we can claim that 
$|\{2(\bx_0-\bmu_1)+\bmu_{12} \}^T\hat{\bB}_2 \bmu_{12}|\le ||2(\bx_0-\bmu_1)+\bmu_{12}||\cdot|| \bmu_{12}||\cdot ||\hat{\bB}_2||=O_P(p^{1/2}||\bmu_{12}||\cdot ||\hat{\bB}_2||)=O_P(p ||\hat{\bB}_2||)$ 
because $||\bmu_{12}||^2=O(p)$ and $||2(\bx_0-\bmu_1)+\bmu_{12}||^2=O_P(p)$.
Thus, (\ref{B.11}) holds under (C-viii) or (C-ix). 
Note that (\ref{B.9}) and (\ref{B.10}) naturally hold when $\hat{\bA}_1=\hat{\bA}_2$ and ${\bA}_1={\bA}_2$. 
Hence, from Lemma B.3, it concludes the result of Proposition 4.2 when $\bx_0\in \pi_1$.

Next, we consider (\ref{B.9}) and the first term of (\ref{B.10}). 
We have that for $l=1,2$ 
\begin{align*}
&|\tr(\bSig_1\hat{\bB}_l)|\le \tr(\bSig_1)||\hat{\bB}_l ||=O_P(p||\hat{\bB}_l||) \ \mbox{and}\\
&|\tr[\{(\bx_0-\bmu_1)(\bx_0-\bmu_1)^T-\bSig_1\}\hat{\bB}_l]|\\
&\le || \bx_0-\bmu_1||^2||\hat{\bB}_l ||+\tr(\bSig_1)||\hat{\bB}_l||=O_P(p||\hat{\bB}_l ||).
\end{align*}

Finally, we consider $\log|\hat{\bA}_l {\bA}_l^{-1} |, l=1,2,$ in (\ref{B.10}). 
Let $\be_{p}$ be an arbitrary (random) $p$-vector such that $||\be_p||=1$. 
Note that $||\be_{p}^T \bA_l^{-1/2}||\in (0,\infty)$ as $p\to \infty$ 
under $\lambda(\bA_l)\in (0,\infty)$ as $p\to \infty$. 
Thus we have that  
$$
\be_{p}^T \bA_l^{-1/2}\hat{\bB}_l\bA_l^{-1/2}\be_{p}=\be_{p}^T \bA_l^{-1/2} \hat{\bA}_l\bA_l^{-1/2}\be_{p}-1=O_P(||\hat{\bB}_l||),
$$
so that $\lambda_{\min}(\bA_l^{-1/2} \hat{\bA}_l\bA_l^{-1/2})-1=O_P(||\hat{\bB}_l||)$ and 
$\lambda_{\max}(\bA_l^{-1/2} \hat{\bA}_l\bA_l^{-1/2})-1=O_P(||\hat{\bB}_l||)$. 
Hence, under $||\hat{\bB}_l||=o_P(1)$, it holds that for $l=1,2$ 
$$
\log|\hat{\bA}_l {\bA}_l^{-1} | =\log|\bA_l^{-1/2} \hat{\bA}_l \bA_l^{-1/2}|=O_P(p|| \hat{\bB}_l||).
$$ 
Note that $\Delta_{\min}=O(p)$ and $\delta_{\min}=O(p)$ under 
$\lambda(\bA_i)\in (0,\infty)$ as $p\to \infty$ for $i=1,2$. 
Then, under (C-viii), it holds that $||\hat{\bB}_l||=o_P(1)$ for $l=1,2$. 
Hence, (C-viii) implies (\ref{B.9}) and  (\ref{B.10}).
It concludes the result of Proposition 4.1 when $\bx_0\in \pi_1$. 
For the case when $\bx_0\in \pi_2$, we can have the same arguments.
The proof is completed. 
\end{proof}
\begin{proof}[Proof of Corollary 4.1]
Under (A-i) we have that $\Var \{ \tr(\bS_{in_i}) \}=O(\tr(\bSig_{i}^2)/n_i)$, $i=1,2$, so that 
$
\tr(\bS_{in_i})=\tr(\bSig_{i})+O_P \{(\tr(\bSig_{i}^2)/n_i)^{1/2}\}. 
$ 
Then, under (C-i') it holds that $\tr(\bS_{in_i})=\tr(\bSig_i)+o_P(\Delta_{\min(II)})=\tr(\bSig_i)\{1+o_P(1)\}$ and $\tr(\bSig_{i}^2)/(n_i p^2)=o(\Delta_{\min(II)}^2/p^2)$
$=o(1)$ for $i=1,2$ because $\Delta_{\min(II)}=O(p)$. 
Thus, we have that under (A-i) and (C-i') 
\begin{align}
||\hat{\bB}_i||&=||\{p/\tr(\bS_{in_i})-p/\tr(\bSigma_i)  \}\bI_p||=\frac{p|\tr(\bS_{in_i})-\tr(\bSigma_i)|}{\tr(\bS_{in_i})\tr(\bSigma_i)}\notag \\
&=O_P\{ (\tr(\bSig_{i}^2)/n_i)^{1/2}/\tr(\bS_{in_i}) \}=o_P\{ \Delta_{\min(II)}/p\}=o_P(1), \label{B.12}
\end{align}
so that $ p||\hat{\bB}_i||=o_P(\Delta_{\min(II)})$. 
Note that $\lambda_{\max}(\bA_i)=\lambda_{\min}(\bA_i)=\tr(\bSig_i)/p\in (0,\infty)$ as $p\to \infty$. 
Thus, from Corollary 2.1 and Proposition 4.1, it concludes the result. 
\end{proof}
\begin{proof}[Proof of Corollary 4.2]
We consider the case when $\bx_0\in \pi_i$. 
Note that $\tr(\bS_{ln_l})/\tr(\bSig_{l})$
$=1+O_P \{(\tr(\bSig_{l}^2)/n_l)^{1/2}/p\}=1+o_P(1)$, $l=1,2,$ and $\tr\{(\bx_0-\bmu_i)(\bx_0-\bmu_i)^T-\bSig_i\}=O_P(\tr(\bSig_i^2)^{1/2})$ under (A-i). 
Also, note that $\tr(\bSig_{i}^2)\tr(\bSig_{l}^2) \le  \lambda_{i1}\lambda_{il}\tr(\bSig_i)\tr(\bSig_l)=o(n_{\min} \delta_{\min(II)}^2p^2 )$, $l=1,2$ under (C-iv').
Then, from (\ref{B.12}), it holds that for $l=1,2$ 
\begin{align}
&\tr[\{(\bx_0-\bmu_i)(\bx_0-\bmu_i)^T-\bSig_i\}\hat{\bB}_l ]\notag \\
&=p\frac{\tr(\bSig_{l})-\tr(\bS_{ln_l})}{\tr(\bSig_{l})\tr(\bS_{ln_l})}
\tr\{(\bx_0-\bmu_i)(\bx_0-\bmu_i)^T-\bSig_i\}\notag \\
&=O_P\{ (\tr(\bSig_{i}^2)\tr(\bSig_{l}^2) /n_l)^{1/2}/p\}=o_P(\delta_{\min(II)}), \ \ \mbox{and} 
\label{B.13}\\
&p||\hat{\bB}_l  ||/n_l^{1/2}=O_P\{\tr(\bSig_{l}^2)^{1/2}/n_l \}=o_P(\delta_{\min(II)})
\notag
\end{align}
under (A-i) and (C-iv'). 
Similarly, from (\ref{B.12}), under (A-i) and (C-iv'), we have that for $i' \neq i$ 
\begin{align*}
&\{2(\bx_0-\bmu_i)+(-1)^{i+1}\bmu_{12} \}^T\hat{\bB}_{i'} \bmu_{12} \\
&=O_P\{(\bmu_{12}^T \bSig_i\bmu_{12}/n_{i'})^{1/2}\}+O_P \{(\tr(\bSig_{i'}^2)/n_{i'})^{1/2}||\bmu_{12} ||^2/p\}\\
&=O_P \{( \lambda_{i1}||\bmu_{12}||^2/n_{i'})^{1/2}\}+O_P \{( \lambda_{i'1}||\bmu_{12}||^2/n_{i'})^{1/2}\}=o_P(\delta_{\min(II)}) 
\end{align*}
from the facts that $\bmu_{12}^T \bSig_i\bmu_{12}\le \lambda_{i1} ||\bmu_{12}||^2$,  
$\tr(\bSig_{i'}^2)=O(\lambda_{i'1}p)$ and $||\bmu_{12}||^2=O(p)$. 
On the other hand, under (A-i) and (C-iv'), from (\ref{B.12}), we have that for $l=1,2$ 
\begin{align*}
\log\{ \tr(\bSig_{l})/\tr(\bS_{ln_l}) \}&=(\tr(\bSig_{l})/\tr(\bS_{ln_l})-1)+O_P\{(\tr(\bSig_{l})/\tr(\bS_{ln_l})-1)^2\}\\
&=(\tr(\bSig_{l})/\tr(\bS_{ln_l})-1)+O_P\{\tr(\bSig_{l}^2)/(n_lp^2)\}\\
&=(\tr(\bSig_{l})/\tr(\bS_{ln_l})-1)+o_P(\delta_{\min(II)}/p)
\end{align*}
from the facts that $\tr(\bSig_{l}^2)/p=O(\lambda_{l1})$ and $\tr(\bSig_{l})/\tr(\bS_{ln_l})=1+o_P(1)$. 
Then, under (A-i) and (C-iv'), it holds that 
$$
\tr(\bSig_i\hat{\bB}_i)-\log|\hat{\bA}_i {\bA}_i^{-1}| 
=p(\tr(\bSig_{i})/\tr(\bS_{in_i})-1)-p \log\{ \tr(\bSig_{i})/\tr(\bS_{in_i}) \}=o_P(\delta_{\min(II)}).
$$
Similarly, under (A-i) and (C-iv'), we have that 
\begin{align}
\notag
\tr(\bSig_i\hat{\bB}_{i'})-\log|\hat{\bA}_{i'}{\bA}_{i'}^{-1}| 
&=p(\tr(\bSig_{i})/\tr(\bSig_{i'})-1)(\tr(\bSig_{i'})/\tr(\bS_{i'n_{i'}})-1)+o_P(\delta_{\min (II)})\\
\label{B.14}
&=O_P(|\tr(\bSig_{i})/\tr(\bSig_{i'})-1|(\tr(\bSig_{i'}^{2})/n_{i'})^{1/2})+o_P(\delta_{\min(II)}).
\end{align}
By combining (\ref{B.13}) to (\ref{B.14}) with Lemma B.3 and Corollary 3.1, we can claim the result.  
\end{proof}
\begin{proof}[Proof of Corollary 4.3]
We can write that 
\begin{align}
s_{in_i(j)}=n_is_{oin_i(j)}/(n_i-1)-n_i(\overline{x}_{ijn_i}-\mu_{ij})^2/(n_i-1),
\label{B.15}
\end{align}
where $s_{oin_i(j)}=\sum_{k=1}^{n_i}(x_{ijk}-\mu_{ij})^2/n_i $. 
Note that 
$\limsup_{p\to \infty} E \{ \exp(t_{ij} | (x_{ijk}-\mu_{ij})^2-\sigma_{i(j)}|/\eta_{i(j)}^{1/2} )\}\le \limsup_{p\to \infty}[ E \{ \exp(t_{ij}|x_{ijk}-\mu_{ij}|^2/\eta_{i(j)}^{1/2} )\}+\exp(t_{ij}\sigma_{i(j)}/\eta_{i(j)}^{1/2} )]<\infty$ under (A-iii).
Then, under (A-iii), for any $x$ satisfying $x\to \infty$ and $x=o(n_i^{1/2})$ as $n_i\to \infty$, we have that as $n_i\to \infty$
$$
P(n_i^{1/2}|s_{oin_i(j)}-\sigma_{i(j)}|/\eta_{i(j)}^{1/2} \ge x )=\exp\Big(-\frac{x^2}{2}\{1+o(1)\}\Big).
$$
Refer to Chapter 6 in \cite{Delapena:2009} for the details of this result. 
Let $ \tau_{1j}= M(\eta_{i(j)}n_i^{-1}\log{p})^{1/2}$ for $j=1,...,p$, where $M>2^{1/2}$. 
Then, under $n_i^{-1}\log{p}=o(1)$, it holds that as $p\to \infty$ 
\begin{align}
\notag
\sum_{j=1}^p
P(|s_{oin_i(j)}-\sigma_{i(j)}|\ge \tau_{1j} )&=
\sum_{j=1}^p
P(n_i^{1/2} |s_{oin_i(j)}-\sigma_{i(j)}|/\eta_{i(j)}^{1/2}\ge M (\log{p})^{1/2} )\\
&=\sum_{j=1}^p\exp\Big(-\frac{M^2 \log{p}}{2}\{1+o(1)\}\Big)
\to 0.
\label{B.16}
\end{align}

Next, we consider the second term of (\ref{B.15}). 
Let $u_{ij}=t_{ij}(\sigma_{i(j)}/\eta_{i(j)})^{1/2}$ for $j=1,...,p$. 
Then, we have that for $j=1,...,p$ 
\begin{align*}
&E \{ \exp (u_{ij}|x_{oijk}|/\sigma_{i(j)}^{1/2} ) \} \\
&=E\{ \exp (u_{ij}|x_{oijk}|/\sigma_{i(j)}^{1/2} )I(|x_{oijk}| \le 1)\}+E \{ \exp (u_{ij}|x_{oijk}|/\sigma_{i(j)}^{1/2} )I(|x_{oijk}|> 1)\}\\
&\le  \exp (u_{ij}/\sigma_{i(j)}^{1/2} )+
E \{  \exp (u_{ij}x_{oijk}^2/\sigma_{i(j)}^{1/2} )\} \le   \exp (u_{ij}/\sigma_{i(j)}^{1/2} )+
E \{  \exp (t_{is}x_{oijk}^2/\eta_{i(j)}^{1/2} )\}, 
\end{align*}
so that $\limsup_{p\to \infty} E \{ \exp (u_{ij}|x_{oijk}|/\sigma_{i(j)}^{1/2} ) \}<\infty$ under (A-iii).  
Thus, in a way similar to (\ref{B.16}), we have that 
\begin{align}
\label{B.17}
\sum_{j=1}^p P(|\overline{x}_{ijn_i}-\mu_{ij}|\ge \tau_{2j} )=
\sum_{j=1}^p P(n_i^{1/2}|\overline{x}_{ijn_i}-\mu_{ij}|/\sigma_{i(j)}^{1/2}\ge M (\log{p})^{1/2} )\to 0
\end{align}
for $ \tau_{2j}= M( \sigma_{i(j)}n_i^{-1}\log{p})^{1/2}$, $j=1,...,p$. 
By combining (\ref{B.16}) and (\ref{B.17}) with (\ref{B.15}), under $n_i^{-1}\log{p}=o(1)$ and (A-iii), we have that 
\begin{align*}
&\sum_{j=1}^pP\{|s_{in_i(j)}- n_i\sigma_{i(j)}/(n_i-1)|\ge n_i(\tau_{1j}+\tau_{2j}^2)/(n_i-1)\} \\
&\le \sum_{j=1}^pP(|s_{oin_i(j)}-\sigma_{i(j)}|+|\overline{x}_{ijn_i}-\mu_{ij}|^2\ge  \tau_{1j}+\tau_{2j}^2) \\
&\le \sum_{j=1}^pP(|s_{oin_i(j)}-\sigma_{i(j)}|\ge \tau_{1j})+
\sum_{j=1}^p P(|\overline{x}_{ijn_i}-\mu_{ij}|^2\ge \tau_{2j}^2)\to 0.
\end{align*}
Note that  $n_i\sigma_{i(j)}/(n_i-1)=\sigma_{i(j)}+o(n_i^{-1/2})$ and $\tau_{2j}^2=o(\tau_{1j})$ under $n_i^{-1}\log{p}=o(1)$. 
Thus we have that $\max_{j=1,...,p}\{|s_{in_i(j)}- \sigma_{i(j)}|\}=O_P( \max_{j=1,...,p}\tau_{1j}) $ under $n_i^{-1}\log{p}=o(1)$ and (A-iii), so that 
\begin{align}
\label{B.18}
\max_{j=1,...,p}\{|s_{in_i(j)}- \sigma_{i(j)}|\}=O_P\{(n_i^{-1}\log{p})^{1/2}\} .
\end{align}
Then, for $i=1,2$, it holds that under $n_i^{-1}\log{p}=o(1)$ 
\begin{align}
\notag
||\hat{\bB}_i||=
||\bS_{i(d)}^{-1}-\bSig_{i(d)}^{-1}||&=\max_{j=1,...,p}\{ |s_{in_i(j)}-\sigma_{i(j)}|/ (s_{in_i(j)} \sigma_{i(j)}) \} \\ 
&=
O_P\{(n_i^{-1}\log{p})^{1/2}\}= o_P(1).
\label{B.19}
\end{align}
Then, 
it follows that (C-i') holds under (4.4). 
From the facts that $\Delta_{\min(III)}=O(p)$, note that $n_{\min}^{-1} \log {p}=o(1)$ under (4.4).
Then, by combining (\ref{B.19}) with Proposition 4.1 and Corollary 2.1, we can claim the result of Corollary 4.3. 
\end{proof}
\begin{proof}[Proofs of Corollary 4.4]
First, note that $s_{n(j)}-\sigma_{(j)}=\sum_{i=1}^2(n_i-1)(s_{in_i(j)}-\sigma_{i(j)})$
$/(\sum_{i=1}^2n_i-2)$.
From (\ref{B.18}), we can claim that $\max_{j=1,...,p}\{|s_{n(j)}- \sigma_{(j)}|\}=O_P\{(n_{\min}^{-1}$
$\log{p})^{1/2}\}$ 
under $n_{\min}^{-1} \log {p}=o(1)$ and (A-iii).
Thus it follows that $||\bS_{n(d)}^{-1}-\bSig_{(d)}^{-1}||=O_P\{(n_{\min}^{-1}\log{p})^{1/2} \}$. 
Note that $\Delta_{(III')}/||\bmu_{12}||^2\in (0,\infty)$ as $p\to \infty$.
Then, by combining Theorem 2.1 with Propositions 2.1 and 4.2, we can claim the result of Corollary 4.4. 
\end{proof}
\begin{proof}[Proofs of Corollary 4.5]
Let $\bS_{oin_i}=\sum_{k=1}^{n_i}(\bx_{ik}-\bmu_{i})(\bx_{ik}-\bmu_{i})^T/n_i$ and denote its $(r,s)$ element by $s_{o in_i(rs)}$ for $r,s=1,...,p$. 
Let $u_{i(rs)}=\min\{t_{ir}/\eta_{i(r)}^{1/2},t_{is}/\eta_{i(s)}^{1/2}\}\eta_{i(rs)}^{1/2}$
for $r,s=1,...,p$. 
Then, we have that for $r,s=1,...,p$ 
\begin{align*}
&E\{\exp (u_{i(rs)} |x_{oirk}x_{oisk}-\sigma_{i(rs)} |/\eta_{i(rs)}^{1/2}) \} \\
&\le E [\exp \{u_{i(rs)} (x_{oirk}^2/2+x_{oisk}^2/2+\sigma_{i(rs)})/\eta_{i(rs)}^{1/2} \}] \\
&\le  \exp( u_{i(rs)}\sigma_{i(rs)}/\eta_{i(rs)}^{1/2})E[ \exp\{t_{ir} x_{oirk}^2/(2\eta_{i(r)}^{1/2})\} \exp\{t_{is}x_{oisk}^2/(2\eta_{i(s)}^{1/2})\}] \\
&\le  \exp( u_{i(rs)}\sigma_{i(rs)}/\eta_{i(rs)}^{1/2})[E\{\exp(t_{ir}x_{oirk}^2/\eta_{i(r)}^{1/2})\}E\{\exp(t_{is}x_{oisk}^2/\eta_{i(s)}^{1/2})\}]^{1/2},
\end{align*}
so that $\limsup_{p\to \infty} E\{\exp (u_{i(rs)} |x_{oirk}x_{oisk}-\sigma_{i(rs)} |/\eta_{i(rs)}^{1/2}) \}<\infty$ under (A-iii). 
Note that $s_{in_i(rs)}=n_is_{oin_i(rs)}/(n_i-1)-n_i(\overline{x}_{irn_i}-\mu_{ir})(\overline{x}_{isn_i}-\mu_{is})/(n_i-1)$, where 
$s_{in_i(rs)}$ is the $(r,s)$ element of $\bS_{in_i}$. 
Also, note that $\eta_{i(rs)}\in(0,\infty)$ as $p\to \infty$ under (A-iii) and 
$\liminf_{p\to \infty} \eta_{i(rs)}>0$ for all $r,s$, 
from the fact that $\eta_{i(rs)}\le  \{(\eta_{i(r)}+\sigma_{i(r)}^2)(\eta_{i(s)}+\sigma_{i(s)}^2) \}^{1/2}$. 
In a way similar to (\ref{B.16}) and (\ref{B.17}), 
under $n_i^{-1}\log{p}=o(1)$, (A-iii) and 
$\liminf_{p\to \infty} \eta_{i(rs)}>0$ for all $r,s$, we have that 
\begin{align*}
&\sum_{r,s=1}^pP\{|s_{in_i(rs)}- n_i\sigma_{i(rs)}/(n_i-1)|\ge n_i(\tau_{1(rs)}+\tau_{2(rs)} )/(n_i-1)\} \\
&\le \sum_{r,s=1}^p\{P(|s_{oin_i(rs)}-\sigma_{i(rs)}|\ge \tau_{1(rs)}) +P(|\overline{x}_{irn_i}-\mu_{ir}||\overline{x}_{isn_i}-\mu_{is}| \ge \tau_{2(rs)} )\}\\
&\le \sum_{r,s=1}^pP(|\overline{x}_{irn_i}-\mu_{ir}|^2+|\overline{x}_{isn_i}-\mu_{sr}|^2 \ge \tau_{2(rs)})\}
+o(1)\to 0
\end{align*}
for $ \tau_{1(rs)}= M(\eta_{i(rs)}n_i^{-1}\log{p})^{1/2}$ and 
$ \tau_{2(rs)}= M^2\{(\sigma_{i(r)}+ \sigma_{i(s)}) n_i^{-1}\log{p}\}$, $r,s=1,...,p$, where $M>2$. 
Thus it holds that $\max_{r,s=1,...,p}\{|s_{in_i(rs)}- \sigma_{i(rs)}|\}=O_P( \max_{r,s=1,...,p}\tau_{1(rs)} )$ because $\tau_{2(rs)}=o(\tau_{1(rs)})$, so that 
\begin{align}
\label{B.20}
\max_{r,s=1,...,p}\{|s_{in_i(rs)}- \sigma_{i(rs)} \}=O_P\{(n_i^{-1}\log{p})^{1/2}\} .
\end{align}
Here, from the equations (A1) and (A2) in \cite{Bickel:2008a}, we have that 
$
||\bM||\le \max_{s=1,...,p}\sum_{t=1}^p|m_{st}| 
$
for any symmetric matrix $\bM$, where $m_{st}$ is the $(s,t)$ element of $\bM$. 
From (\ref{B.20}), we have that 
\begin{align}
||\bS_{in_i}-\bSig_i||=O_P\{p(n_i^{-1} \log{p})^{1/2}\}=o_P(1) \label{B.21}
\end{align}
under $n_i^{-1} p^2 \log{p} =o(1)$, (A-iii) and $\liminf_{p\to \infty} \eta_{i(rs)}>0$ for all $r,s$.
Then, under $\lambda(\bSig_i)\in (0,\infty)$ as $p\to \infty$, we can claim that $\lambda(\bS_{in_i})\in (0,\infty)$ in probability. 
Thus it holds that $||\be_{p}^T\bSig_{i}^{-1}||\in (0, \infty)$ and $||\be_{p}^T\bS_{in_i}^{-1}||\in (0, \infty)$ in probability, where $\be_{p}$ is an arbitrary (random) $p$-vector such that $||\be_p||=1$. 
Then, from (\ref{B.21}), we have that $\be_{p}^T\bSig_{i}^{-1}(\bS_{in_i}-\bSig_i)\bS_{in_i}^{-1}\be_{p}=\be_{p}^T(\bSig_i^{-1}-\bS_{in_i}^{-1})\be_{p}=O_P\{p(n_i^{-1} \log{p})^{1/2}\}$  under $n_i^{-1} p^2 \log{p} =o(1)$, (A-iii) and $\liminf_{p\to \infty} \eta_{i(rs)}>0$ for all $r,s$, so that 
$
||\hat{\bB}_i||=O_P\{p(n_i^{-1} \log{p})^{1/2}\}=o_P(1). 
$
Note that (C-i') and (C-ii') hold under the conditions of Corollary 4.5. 
Also, note that $\tr\{(\bI_p-\bSig_i\bSig_{i'}^{-1})^2\}=O(p)$ ($i'\neq i$) under $\lambda(\bSig_i)\in (0,\infty)$ as $p\to \infty$. 
By combining Corollary 3.2 with Proposition 4.1, we can claim the result of Corollary 4.5. 
\end{proof}
\begin{proof}[Proof of Corollary 5.1]
By using Theorem 5.1, we can claim the result straightforwardly. 
\end{proof}
\begin{proof}[Proof of Corollary 5.2]
Let us write that for $i=1,2$  
$$
W_i( \bSig_{i(d)}^{-1} )_{FS}=\sum_{j\in \mbox{{\footnotesize $\bD$}}} \{(x_{0j}-\overline{x}_{ijn_i})^2/{\sigma}_{i(j)}-s_{in_i(j)}/({\sigma}_{i(j)} n_i)+\log{{\sigma}_{i(j)} }\}.
$$
Note that $E\{W_{i'}( \bSig_{i'(d)}^{-1} )_{FS}\}-E\{W_i( \bSig_{i(d)}^{-1} )_{FS}\}=\Delta_{i(III)}$ ($i' \neq i$) when $\bx_0 \in \pi_i$. 
Also note that $\liminf_{p\to \infty} \Delta_{\min(III)}/p_*>0$ under $\liminf_{p\to \infty} \theta_j>0$ for all $j \in \bD$. 
If $\lambda_{\max}(\bSig_{i*})=o(p_*)$, (C-i') and (C-ii') hold for $\bSig_{i*},\ i=1,2$. 
Here, 
let us write that $\bSig_{i(d)*}=\mbox{diag}(\sigma_{i(j_1)},...,\sigma_{i(j_{p_*})})$ and 
$\bS_{i(d)*}=\mbox{diag}(s_{in_i({j}_1)},...,s_{in_i({j}_{p_*})})$ for $i=1,2$, 
where $\bD=\{j_1,....,j_{p_*}\}$. 
Then, in a way similar to (\ref{B.19}), under $n_i^{-1}\log{p}=o(1)$ and (A-iii), it holds that
$
||\bS_{i(d)*}^{-1} -\bSig_{i(d)*}^{-1}||=O_P\{(n_i^{-1}\log{p})^{1/2}\}.
$
Hence, we have that $p_*||\bS_{j(d)*}^{-1} -\bSig_{j(d)*}^{-1}||=o_P(\Delta_{\min(III)})$ under $\liminf_{p\to \infty} \theta_j>0$ for all $j\in \bD$. 
By combining Corollary 5.1 with Propositions 2.1 and 4.1, we can claim the result.
\end{proof}

\section*{Acknowledgements}
Research of the first author was partially supported by 
Grants-in-Aid for Scientific Research (A) and Challenging Exploratory Research, Japan Society for the Promotion of Science (JSPS), under Contract Numbers 15H01678 and 26540010.
Research of the second author was partially supported by 
Grant-in-Aid for Young Scientists (B), JSPS, under Contract Number 26800078.


\begin{thebibliography}{}


\bibitem[\protect\citeauthoryear{Aoshima and Yata}{2011}]{Aoshima:2011}
Aoshima, M. and Yata, K. (2011).
Two-stage procedures for high-dimensional data. 
\textit{ Sequential Anal. ({\it Editor's special invited paper})}, \textbf{30}, 356--399. 


\bibitem[\protect\citeauthoryear{Aoshima and Yata}{2014}]{Aoshima:2014}
Aoshima, M. and Yata, K. (2014).
A distance-based, misclassification rate adjusted classifier for multiclass, high-dimensional data. 
\textit{Ann. Inst. Statist. Math.}, \textbf{66}, 983--1010. 


\bibitem[\protect\citeauthoryear{Armstrong et al.}{2002}]{Armstrong:2002}
Armstrong, S.A., Staunton, J.E., Silverman, L.B., Pieters, R., den Boer, M.L., Minden, M.D., Sallan, S.E., 
Lander, E.S., Golub, T.R. and Korsmeyer, S.J. (2002).
MLL translocations specify a distinct gene expression profile that distinguishes a unique leukemia. 
\textit{Nature Genetics}, \textbf{30}, 41--47. 

\bibitem[\protect\citeauthoryear{Bai and Saranadasa}{1996}]{Bai:1996}
Bai, Z. and Saranadasa, H. (1996).
Effect of high dimension: by an example of a two sample problem. 
\textit{Statist. Sinica}, \textbf{6}, 311--329. 


\bibitem[\protect\citeauthoryear{Bickel and Levina}{2004}]{Bickel:2004}
Bickel, P.J. and Levina, E. (2004).
Some theory for Fisher's linear discriminant function, `naive Bayes', and some alternatives when there are many more variables than observations. 
\textit{Bernoulli}, \textbf{10}, 989--1010. 

\bibitem[\protect\citeauthoryear{Bickel and Levina}{2008a}]{Bickel:2008a}
Bickel, P.J. and Levina, E. (2008a). 
Regularized estimation of large covariance matrices.
\textit{Ann. Statist.}, \textbf{36}, 199--227. 

\bibitem[\protect\citeauthoryear{Bickel and Levina}{2008b}]{Bickel:2008}
Bickel, P.J. and Levina, E. (2008b).
Covariance regularization by thresholding.
\textit{Ann. Statist.}, \textbf{36}, 2577--2604. 


\bibitem[\protect\citeauthoryear{Cai, Liu and Luo}{2011}]{Cai:2011a}
Cai, T.T., Liu, W. and Luo, X. (2011).
A constrained $\ell_1$ minimization approach to sparse precision matrix estimation. 
\textit{J. Amer. Statist. Assoc.}, \textbf{106}, 594--607. 


\bibitem[\protect\citeauthoryear{Cai and Liu}{2011}]{Cai:2011b}
Cai, T.T. and Liu, W. (2011).
A direct estimation approach to sparse linear discriminant analysis. 
\textit{J. Amer. Statist. Assoc.}, \textbf{106}, 1566--1577. 


\bibitem[\protect\astroncite{Chan and Hall}{2009}]{Chan:2009}
Chan, Y.-B. and Hall, P. (2009). 
Scale adjustments for classifiers in high-dimensional, low sample size settings.
\textit{Biometrika}, \textbf{96}, 469--478. 

\bibitem[\protect\citeauthoryear{de la Pe\~{n}a, Lai and Shao}{2009}]{Delapena:2009}
de la Pe\~{n}a, V.H., Lai, T.L. and Shao, Q.M. (2009). 
{\it Self-Normalized Processes}. Berlin: Springer-Verlag. 


\bibitem[\protect\citeauthoryear{Dudoit, Fridlyand and Speed}{2002}]{Dudoit:2002}
Dudoit, S., Fridlyand, J. and Speed, T.P. (2002).
Comparison of discrimination methods for the classification of tumors using gene expression data. 
\textit{J. Amer. Statist. Assoc.}, \textbf{97}, 77--87. 


\bibitem[\protect\citeauthoryear{Fan and Fan}{2008}]{Fan:2008}
Fan, J. and Fan, Y. (2008).
High-dimensional classification using features annealed independence rules. 
\textit{Ann. Statist.}, \textbf{36}, 2605--2637.


\bibitem[\protect\citeauthoryear{Fan, Feng and Tong}{2012}]{Fan:2012}
Fan, J., Feng, Y. and Tong, X. (2012). 
A road to classification in high dimensional space: the regularized optimal affine discriminant.
\textit{J. R. Statist. Soc. Ser. B}, \textbf{74}, 745--771. 


\bibitem[\protect\citeauthoryear{Golub et al.}{1999}]{Golub:1999}
Golub, T.R., Slonim, D.K., Tamayo, P., Huard, C., Gaasenbeek, M., Mesirov, J.P., Coller, H., Loh, M.L., Downing, J.R., Caligiuri, M.A., Bloomfield, C.D. and Lander, E.S. (1999).
Molecular classification of cancer: class discovery and class prediction by gene expression monitoring.
\textit{Science}, \textbf{286}, 531--537. 

\bibitem[\protect\citeauthoryear{Hall, Marron and Neeman}{2005}]{Hall:2005}
Hall, P., Marron, J.S. and Neeman, A. (2005).
Geometric representation of high dimension, low sample size data.
\textit{J. R. Statist. Soc. Ser. B}, \textbf{67}, 427--444.


\bibitem[\protect\citeauthoryear{Huang, Tong and Zhao}{2010}]{Huang:2010}
Huang, S., Tong, T. and Zhao, H. (2010).
Bias-corrected diagonal discriminant rules for high-dimensional classification. 
\textit{Biometrics}, \textbf{66}, 1096--1106. 


\bibitem[\protect\citeauthoryear{Li and Shao}{2015}]{Li:2015}
Li, Q. and Shao, J. (2015).
Sparse quadratic dicriminant analysis for high dimensional data. 
\textit{Statist. Sinica}, in press (doi: 10.5705/ss.2013.150).


\bibitem[\protect\citeauthoryear{Marron, Todd and Ahn}{2007}]{Marron:2007}
Marron, J.S., Todd, M.J. and Ahn, J. (2007).
Distance-weighted discrimination. 
\textit{J. Amer. Statist. Assoc.}, \textbf{102}, 1267--1271. 


\bibitem[\protect\citeauthoryear{McLeish}{1974}]{McLeish:1974}
McLeish, D.L. (1974). 
Dependent central limit theorems and invariance principles. 
\textit{Ann. Probab.}, \textbf{2}, 620--628. 



\bibitem[\protect\citeauthoryear{Shao et al.}{2011}]{Shao:2011}
Shao, J., Wang, Y., Deng, X. and Wang, S. (2011). 
Sparse linear discriminant analysis by thresholding for high dimensional data. 
\textit{Ann. Statist.}, \textbf{39}, 1241--1265. 


\bibitem[\protect\citeauthoryear{Tan et al.}{2005}]{Tan:2005}
Tan, A.K, Naiman, D.Q., Xu, L., Winslow, R.L. and Geman, D. (2005).
Simple decision rules for classifying human cancers from gene expression profiles. 
\textit{Bioinformatics}, \textbf{21}, 3896--3904. 

\bibitem[\protect\citeauthoryear{Vapnic}{1999}]{Vapnic:1999}
Vapnic, V.N. (1999).
\textit{The Nature of Statistical Learning Theory (second ed.)}.
New York: Springer-Verlag.


\bibitem[\protect\citeauthoryear{Yata and Aoshima}{2013}]{Yata:2013b}
Yata, K. and Aoshima, M. (2013).
PCA consistency for the power spiked model in high-dimensional settings. 
\textit{J. Multivariate Anal.}, \textbf{122}, 334--354. 



\end{thebibliography}
\end{document}